\documentclass{article}

    \PassOptionsToPackage{numbers, sort&compress}{natbib}

    \usepackage[final]{neurips_2020}

\usepackage[utf8]{inputenc} 
\usepackage[T1]{fontenc}    
\usepackage{hyperref}       
\usepackage{url}            
\usepackage{booktabs}       
\usepackage{amsfonts}       
\usepackage{nicefrac}       
\usepackage{microtype}      
\usepackage[pdftex]{graphicx}
\usepackage{amsthm}
\usepackage{xcolor}
\usepackage{algorithm}
\usepackage{subcaption}
\usepackage[noend]{algpseudocode}
\usepackage{appendix}

\makeatletter
\newcommand{\printfnsymbol}[1]{%
  \textsuperscript{\@fnsymbol{#1}}%
}
\makeatother

\usepackage{enumerate}

\title{Adaptive Learning of Rank-One Models for \\Efficient Pairwise Sequence Alignment}

\author{%
  Govinda M. Kamath\thanks{Equal Contributors}$\ \ ^1$, Tavor Z. Baharav \printfnsymbol{1}$^2$, and Ilan Shomorony$\ ^3$ \\
  $^1$Microsoft Research New England, Cambridge, MA\\
  $^2$Department of Electrical Engineering, Stanford University,
  Stanford, CA\\
  $^3$Department of Electrical and Computer Engineering,\\
  University of Illinois, Urbana-Champaign, IL\\
  \texttt{gokamath@microsoft.com, tavorb@stanford.edu, ilans@illinois.edu} \\
}

\usepackage{amsmath} 
\usepackage{nicefrac}       
\usepackage{microtype}      
\usepackage{amsthm}
\usepackage{amsfonts}
\usepackage{dsfont}
\usepackage{amssymb}

\newcommand{\x}{{\mathbf x}}

\newcommand{\p}{{\mathbf p}}
\newcommand{\q}{{\mathbf q}}

\newcommand{\uvec}{{\mathbf u}}
\newcommand{\vvec}{{\mathbf v}}

\renewcommand{\u}{\uvec}
\renewcommand{\v}{\vvec}

\newcommand{\al}[1]{\begin{align}#1\end{align}}
\newcommand{\aln}[1]{\begin{align*}#1\end{align*}}

\newcommand\one{{\mathds{1}}}

\newcommand\1{{\mathbf{1}}}

\newcommand\Ber{{\rm Ber}}

\renewcommand{\top}{T}
\newcommand{\xor}{\oplus}

\newcommand{\js}{{\rm JS}}

\newcommand{\ignore}[1]{}

\usepackage{semantic}
\mathlig{==}{\equiv}
\mathlig{=.}{\doteq}
\mathlig{:=}{\triangleq}
\mathlig{<<}{\ll}
\mathlig{>>}{\gg}
\mathlig{<>}{\neq}
\mathlig{<=}{\leq}
\mathlig{>=}{\geq}
\mathlig{<==}{\Leftarrow}
\mathlig{==>}{\Rightarrow}
\mathlig{<=>}{\Leftrightarrow}
\mathlig{<==>}{\iff}
\mathlig{<--}{\leftarrow}
\mathlig{-->}{\rightarrow}
\mathlig{<->}{\leftrightarrow}
\mathlig{+-}{\pm}
\mathlig{-+}{\mp}
\mathlig{...}{\dots}
\mathlig{!=}{\stackrel{!}{=}}

\DeclareMathOperator*{\argmax}{\arg\!\max}

\newif\ifshowanswer    

\newcommand{\isitthree}[1]
{
  \ifnum#1=3
    number #1 is 3
  \else
    number #1 is not 3
  \fi
}

\newcommand{\be}{\begin{equation}}
\newcommand{\ee}{\end{equation}}

\newcommand\R{{\mathbb{R}}}

\renewcommand\P{\Pr} 
\newcommand\E{{\mathbb{E}}}

\renewcommand\Pr{{\mathbb P }}

\newcommand\CI{{\mathcal I}}

\newcommand\ep{{\epsilon}}
\renewcommand\P{{\mathbb{P}}}

\newtheorem{thm}{Theorem}

\newtheorem{lem}{Lemma}

\newtheorem{note}{Remark}

\begin{document}

\maketitle

\begin{abstract}
Pairwise alignment of DNA sequencing data is a ubiquitous task in bioinformatics and typically represents a heavy computational burden. State-of-the-art approaches to speed up this task use hashing to identify short segments ($k$-mers) that are shared by pairs of reads, which can then be used to estimate alignment scores. 
However, when the number of reads is large, accurately estimating alignment scores for all pairs is still very costly.
Moreover, in practice, one is only interested in identifying pairs of reads with large alignment scores.
In this work, we propose a new approach to pairwise alignment estimation based on two key new ingredients.
The first ingredient is to cast the problem of pairwise alignment estimation under a general framework of rank-one crowdsourcing models, where the workers' responses correspond to $k$-mer hash collisions.
These models can be accurately solved via a spectral decomposition of the response matrix.
The second ingredient is to utilise a multi-armed bandit algorithm to adaptively refine this spectral estimator only for read pairs that are likely to have large alignments. 
The resulting algorithm 
iteratively performs a spectral decomposition of the response matrix for 
adaptively chosen subsets of the read pairs.
\end{abstract}

\section{Introduction} \label{sec:introduction}

A key step in many bioinformatics analysis pipelines is the identification of regions of
similarity between pairs of DNA sequencing reads.
This task, known as \emph{pairwise sequence
alignment}, is a heavy computational burden, particularly in the context of third-generation long-read sequencing technologies, which produce noisy reads \citep{pacbio_errorRate}. 
This challenge is commonly addressed via a two-step approach: first, 
an alignment estimation procedure is used to 
identify those
pairs that are likely to have a large alignment.
Then, computationally intensive alignment algorithms are applied only to the selected pairs.
This two-step approach can greatly speed up the alignment task because, in practice, one  only cares about the
alignment between reads with a large sequence identity or overlap.

Several works have
developed ways to efficiently
estimate pairwise alignments
\citep{Myers2014,Berlin2015,li2016minimap,li2018minimap2}.
The proposed algorithms typically rely on hashing to efficiently find pairs of reads that share many $k$-mers (length-$k$ contiguous substrings).
Particularly relevant to our discussion is the MHAP algorithm of \citet{Berlin2015}.
Suppose we want to estimate the overlap size between two strings $S_0$ and $S_1$ and let $\Gamma(S_i)$ be the set of all $k$-mers in $S_i$, $i=0,1$.
For a hash function $h$, we can compute a min-hash
\begin{equation}
    h(S_i) :=
    \min \{h(x) :  x \in \Gamma(S_i)\},
\end{equation}
for each read $S_i$.
The key observation behind MHAP is that, for a randomly selected hash function $h$,
\begin{align}
    \Pr\left[ h(S_0) =  h(S_1)\right] = \frac{|\Gamma(S_0) \cap \Gamma(S_1)|}{|\Gamma(S_0) \cup \Gamma(S_1)|}.
    \label{eq:prjk}
\end{align}
In other words, the indicator function $\one\{h(S_0) = h(S_1)\}$ provides an unbiased estimator for the $k$-mer \emph{Jaccard similarity} between the sets $\Gamma(S_0)$ and $\Gamma(S_1)$, which we denote by $\js_k(S_0,S_1)$.
By computing $\one\{h(S_0) = h(S_1)\}$ for several different random hash functions, one can thus obtain an arbitrarily accurate estimate of $\js_k(S_0,S_1)$.
As discussed in \cite{Berlin2015}, 
$\js_k(S_0,S_1)$ serves as an estimate for the overlap size and can be used to filter pairs of reads that are likely to have a significant overlap.

Now suppose we fix a reference read $S_0$ and wish to estimate the size of the overlap between $S_0$ and $S_i$, for $i=1,...,n$.
Assume that all reads are of length $L$ and let $p_i \in [0,1]$ be the overlap fraction between $S_i$ and $S_0$ (i.e., the maximum $p$ such that a $pL$-prefix of $S_i$ matches a $pL$-suffix of $S_0$ or vice-versa).
By taking $m$ random hash functions $h_1,...,h_m$, we can compute min-hashes $h_j(S_i)$ for $i=0,1,...,n$ and $j=1,...,m$.
The MHAP approach corresponds to estimating each $p_i$ as 
\al{
\hat p_i = \frac1m \sum\nolimits_{j=1}^m \one\{h_j(S_0) = h_j(S_i)\}.
\label{eq:mhap}
}
In the context of crowdsourcing and vote aggregation \citep{dawid1979maximum,raykar2010learning,ghosh2011moderates},
one can think of each hash function $h_j$ as a worker/expert/participant, who is providing binary responses $Y_{i,j} = \one\{h_j(S_0) = h_j(S_i)\}$ to the questions ``do $S_i$ and $S_0$ have a large alignment score?'' for $i=1,...,n$.
Based on the binary matrix of observations $Y = [Y_{i,j}]$, we want to estimate the true overlap fractions $p_1,...,p_n$.

The idea of jointly estimating $p_1,...,p_n$ from the whole matrix $Y$ was recently proposed by \citet{baharav2019spectral}.
The authors noticed that in practical datasets the distribution of $k$-mers can be heavily skewed.
This causes some hash functions $h_j$ to be ``better than others'' at estimating alignment scores.
Hence, much like in crowdsourcing models, each worker has a different level of expertise, which determines the quality of their answer to all questions.
Motivated by this, \citet{baharav2019spectral} proposed a model where each hash function $h_j$ has an associated unreliability parameter $q_j \in [0,1]$ and, for $i=1,...,n$ and $j=1,...,m$, the binary observations are modeled as 
\al{
Y_{i,j} \sim  \Ber(p_i) \vee \Ber(q_j),
\label{eq:pqmodel}
}
where $\Ber(p)$ is a Bernoulli distribution with parameter $p$ and $\vee$ is the OR operator.
If a given $h_j$ assigns low values to common $k$-mers, spurious min-hash collisions are more likely to occur, 
leading to 
the observation $Y_{i,j} = 1$ when $S_i$ and $S_0$ do not have an overlap (thus being a ``bad'' hash function).
Similarly, some workers in crowdsourcing applications provide less valuable feedback, but we cannot know a priori how reliable each worker is.

A key observation 
about the model in \eqref{eq:pqmodel} is that, in expectation, the observation matrix $Y$ is rank-one after accounting for an offset.
More precisely, since $\E Y_{i,j} = p_{i} + q_{j}- p_{i}q_{j} = (1-p_i)(q_j-1) + 1$,
\begin{align}
    \E Y -\1 \1^\top= (\1-\p )(\q-\1)^\top,
\end{align}
where $\p = [p_1,...,p_m]^\top$ and $\q = [q_1,...,q_n]^\top$.
\citet{baharav2019spectral} proposed to estimate $\p$ by computing a singular value decomposition (SVD) of $Y - \1 \1^\top$, and setting $\hat \p = \1 - \uvec$, where $\uvec$ is the leading left singular vector of $Y - \1 \1^\top$.
The resulting overlap estimates $\hat{p}_1,...,\hat{p}_n$ are called the \emph{Spectral Jaccard Similarity} scores and were shown to provide a much better estimate of overlap sizes than the estimator given by \eqref{eq:mhap},
by accounting for the variable quality of hash functions for the task.

In this paper, motivated by the model of \citet{baharav2019spectral}, we consider the more general framework of rank-one models.
In this setting, a vector of parameters $\uvec = [u_1,...,u_n]^T$ (the item qualities) is to be estimated from the binary responses provided by $m$ workers, and the $n \times m$ observation matrix $X$ is assumed to satisfy $\E X = \u \v^T$.
In the context of these rank-one models, a natural estimator for $\uvec$ is the leading left singular vector of $X$.
Such a spectral estimator has been shown to have good performance both in the context of pairwise sequence alignment \citep{baharav2019spectral} and in voting aggregation applications \cite{ghosh2011moderates,karger2013efficient}.
However, the spectral decomposition by default allocates worker resources uniformly across all items. 
In practice, one is often only interested in identifying the ``most popular'' items, which, in the context of pairwise sequence alignment, corresponds to the reads $S_i$ that have the largest overlaps with a reference read $S_0$.
Hence, we seek strategies that can harness the performance of spectral methods while using adaptivity to avoid wasting worker resources on unpopular items. 

\textbf{Main contributions:}
We propose an adaptive spectral estimation algorithm, based on multi-armed bandits, for identifying the $k$ largest entries of the leading left singular vector $\u$  of $\E X$.
A key technical challenge is that multi-armed bandit algorithms generally rely on our ability to build confidence intervals for each arm, but it is difficult to obtain tight \emph{element-wise} confidence intervals for the singular vectors of random matrices with low expected rank \citep{abbe2017entrywise}.
For that reason, we propose a variation of the spectral estimator for $\uvec$, in which one computes the leading \emph{right} singular vector first and uses it to estimate the entries of $\u$ via a \emph{matched filter} \cite{verdu1998multiuser}.
This allows us to compute entrywise confidence intervals for each $u_i$, 
which in turns allows us to adapt the Sequential Halving bandit algorithm of \citet{Karnin2013} to identify the top-$k$ entries of $\u$.
We provide theoretical performance bounds on the total workers' response budget required to correctly identify the top-$k$ items with a given probability.
We empirically validate our algorithm on controlled experiments that simulate the vote aggregation scenario and on real data in the context of pairwise alignment of DNA sequences.
For a PacBio \mbox{\emph{E.~coli}} dataset \cite{PBecoli}, we show that adaptivity can reduce the budget requirements (which correspond to the number of min-hash comparisons) by around half.

\textbf{Related work:}
Our work is motivated by the bioinformatics literature on pairwise sequence alignment \citep{Myers2014,Berlin2015,li2016minimap,li2018minimap2}.
In particular, we build on the idea of using min-hash-based techniques to efficiently estimate pairwise sequence alignments \cite{Berlin2015}, described by the estimator in (\ref{eq:mhap}).
More sophisticated versions of this idea have been proposed, such as the use of a \emph{tf-idf} (term frequency-inverse document frequency) weighting for the different hash functions \cite{chum2008near,marccais2019locality}, which 
follows the same observation 
that ``some hashes are better than others''
made by \citet{baharav2019spectral}.
A proposed strategy to reduce the number of hash functions needed for the alignment estimates is to use 
\emph{bottom sketches} \cite{ondov2016mash}.
More precisely, for a single hash function, one can compute $s$ minimisers per read (the bottom-$s$ sketch) and estimate the alignments based on the size of the intersection of bottom sketches.
\emph{Winnowing} has also been used in combination with min-hash techniques to allow the mapping of reads to very long sequences, such as full genomes \cite{jain2017fast}.

The literature on crowdsourcing and vote aggregation is vast \citep{raykar2010learning,dalvi2013aggregating, whitehill2009whose,ghosh2011moderates, karger2014budget, liu2012variational, karger2013efficient, whitehill2009whose,  welinder2010multidimensional, zhou2014aggregating, zhou2012learning, zhang2014spectral, shah2016permutation}. 
Many of these works are motivated by the classical Dawid-Skene model \citep{dawid1979maximum}. 
\citet{ghosh2011moderates} considered a setting where the
questions have binary answers and the workers' responses are noisy.
They proposed a rank-one model and
used a spectral method to estimate the true answers.
For a similar setting, \citet{karger2014budget} showed that random allocation
of questions to workers (according to a sparse random
graph) followed by belief propagation is 
optimal to obtain answers to the questions with some
probability of error, and this work was later extended 
\citep{karger2013efficient, dalvi2013aggregating}. 
While the rank-one model these works have considered is similar in spirit to the 
one we consider, in our setting, the \emph{true} answers to the questions are not binary.
Rather, our answers are parameters in $(0,1)$, which can be thought of as the
quality of an item or the 
fraction of the population that would answer ``yes'' to a question.
Natural tasks to consider 
in our case 
would be to find
the top ten products among a catalogue of 100,000 products,
or identify the items that are liked by more than 80\% of the
population. 
Notice that such questions are meaningless in the case of questions with binary answers.

The use of adaptive strategies for ranking a set of items or identifying the top-$k$ items have been studied in the context of \emph{pairwise comparisons} between the items  \citep{heckel2018approximate,heckel2019active,szorenyi2015online, busa2013top}.
Moreover, the top-$k$ arm selection problem has been considered in many contexts 
and a large set of algorithms exist to identify the 
best items while minimising the
number of observations required to do that \citep{maron1994hoeffding, even2002pac, heidrich2009hoeffding, bubeck2013multiple, kalyanakrishnan2012pac, jun2016top, Karnin2013}.
Recent works have taken advantage of some of these multi-armed bandit algorithms 
to solve large-scale computation problems 
\citep{Bagaria2017,Bagaria2018, baharav2019ultra},
which is similar in flavor to our work.

\textbf{Outline of manuscript: }
In Section \ref{sec:model},
we introduce rank-one models and give examples of potential applications. 
In Section \ref{sec:analysis}, we develop a spectral estimator for
these models that enables construction of confidence intervals. 
In Section \ref{sec:adaptive}, we leverage these confidence intervals to
develop adaptive bandit-based algorithms. Section \ref{sec:results}
presents empirical results.

\section{Rank-One Models}
\label{sec:model}

While our main target application is the pairwise sequence alignment problem, we define rank-one models for the general setting of response aggregation problems.
In this setting, 
we are interested in estimating a set of parameters, or item values,  
$u_1,...,u_n$.
To do that, we recruit a set of workers with unknown levels of expertise, which provide binary opinions about the items.
We can choose the number of workers and their opinions can be requested for any 
subset of the $n$ items.
The rank-one model assumption is that the matrix of responses $X = [X_{i,j}]$ is rank-one in expectation; i.e.,
\begin{align}
    \E X = \uvec \vvec^\top
\end{align}
for unknown vectors $\uvec$ and $\vvec$ with entries in $(0,1)$.
This means that, for each $i$ and $j$, $X_{i,j}$ has a $\Ber(u_iv_j)$ distribution.
Furthermore, we assume throughout that all entries $X_{i,j}$ are independent.

Abstractly, this setting applies to a situation in which we have 
a list of questions that we want answered. 
One can think of 
a question as the rating of a movie or whether a product is liked.
The parameter associated with the $i$th question, $u_i$, can be thought
of as representing the average rating of the movie, or the
fraction of people in the population who like the product.
We can think of $X_{i,j}$ as a noisy version of 
the response of
worker $j$ to question $i$, with the noise 
modelling respondent error. 
We point out that $X$ can alternatively be viewed
as noisy observations of responses to the questions
through a \emph{binary channel} \citep{cover2012elements}, with each worker
having a channel of distinct characteristics. 
We next discuss two instances of this model and the corresponding binary channels.

\begin{figure}[t]
\begin{center}
\centerline{\includegraphics[width=0.7\columnwidth]{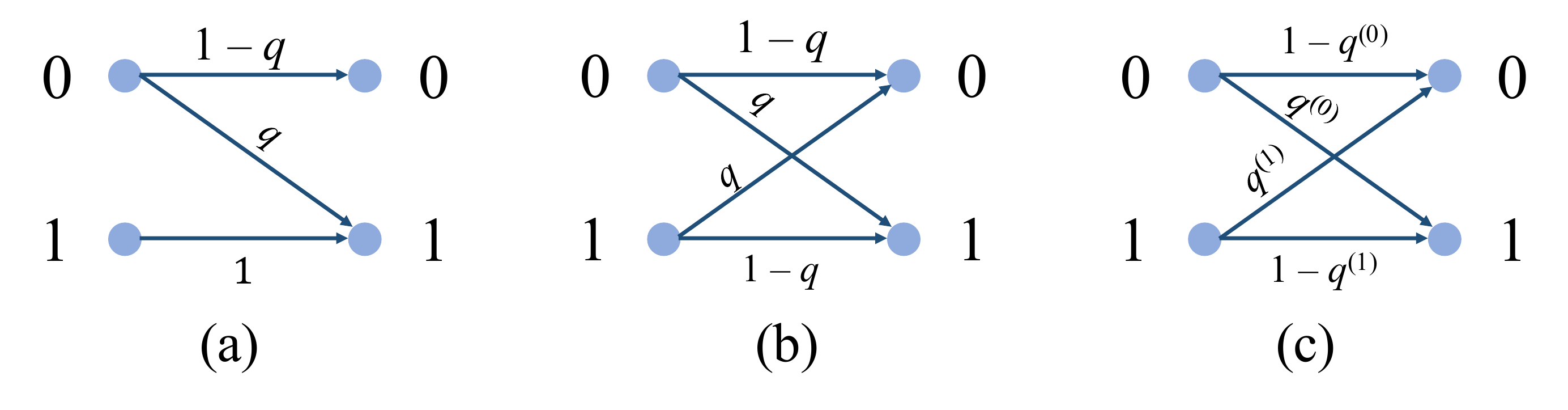}}
\end{center}
\vskip -0.2in
\caption{\emph{Three Channel Models}: (a) the $Z$-channel models one-sided errors; (b) the binary symmetric channel models two-sided symmetric errors; (c) the general binary channel admits different probabilities $q^{(0)}$ for $0 \to 1$ errors and $q^{(1)}$ for $1 \to 0$ errors.
}
\label{fig:channels}
\vspace{-.3cm}
\end{figure}

{\bf One-sided error:}
In the pairwise sequence alignment problem presented in 
Section \ref{sec:introduction}, we want to recover the 
pairwise overlap fractions (or alignment scores) $p_i \in (0,1)$
between read $S_0$ and read $S_i$, for $i=1,...,n$.
The observation for read pair $(S_0,S_i)$ and hash function $h_j$
is modelled by $Y_{i,j} \sim  \Ber(p_i) \vee \Ber(q_j)$,
which can be thought of as observing a 
$\Ber(p_i)$ random variable through a \emph{Z-channel} with 
crossover probability $q_j$. 
A {Z-channel}, shown in Fig.~\ref{fig:channels}(a), is one where an input $1$
is never flipped but an input $0$ is flipped with some
probability \citep{cover2012elements}. 
Hence, this models one-sided errors. 
If we process the data as $X := Y-\1\1^T$,  we then have that
\begin{align}
    \mathbb{E}[X] = (\1 - \p )(\q -\1)^T,
\end{align}
giving us $\uvec =\1 - \p$ and $\vvec = \1 - \q$.
While, in this case, the entries $X_{i,j}$ are technically in $\{0,-1\}$, our main results (presented in Section~\ref{sec:analysis} for a binary matrix $X \in \{0,1\}^{n\times m}$) can be readily extended.

{\bf Two-sided error:}
In the binary crowdsourcing problem, $n$ items have associated 
parameters $p_i \in (0,1)$ (the population rating of the item), for $i=1,...,n$.
The observation of the rating of worker $j$ on the $i$th item is 
modelled as 
$
    Y_{i,j} \sim  \Ber(p_i) \xor \Ber(q_j),
$
where $\xor$ represents the XOR operation. 
This can be thought of as observing a 
$\Ber(p_i)$ random variable through a \emph{Binary Symmetric Channel} with 
crossover probability $q_j$. 
A {Binary Symmetric Channel}, shown in Fig.~\ref{fig:channels}(b), is one where the probability of flipping $1$
to $0$ and $0$ to $1$ is the same.
The processed data in this case is $X := Y-\frac{1}{2}\1\1^T$, and the expected value of our observation matrix is given by
\begin{align}
    \mathbb{E}[X] = \left(\p - \tfrac{1}{2}\1 \right)(\1 - 2\q)^T,
    \label{eq:symmetricmodel}
\end{align}
giving us $\uvec =\p - \frac{1}{2}\1$ and $\vvec = \1 - 2\q$.
As in the case of one-sided errors, the observations $X_{i,j}$ are not in $\{0,1\}$ but they only take two values and the results in Section~\ref{sec:analysis} still hold.

Notice that if we do not make the assumption of symmetry and use a
general binary channel, shown in Fig.~\ref{fig:channels}(c),
the model is still low-rank (it is rank-$2$) in expectation. 
In this manuscript, we focus on rank-one models, but briefly
discuss this generalization in Appendix~\ref{app:rank2}.

{\bf Model Identifiability:}
Strictly speaking, the models described above are not identifiable,
unless extra information is provided. 
This is because $\|\u\|$ and $\|\v\|$
are unspecified, and replacing $\uvec$ and $\vvec$ with $\alpha \uvec$ and $\frac1\alpha \vvec$ leads to the same distribution for the observation matrix $X$.

In practice one can overcome this issue by including questions with 
known answers. For the pairwise sequence alignment problem of
\citet{baharav2019spectral}, the authors add
``calibration reads'' to the dataset. 
These are random reads, expected to have zero overlap with other reads in the 
dataset. In a crowdsourcing setting one similarly can
add questions whose answers are known to the dataset. 
Based on questions with known answers, it is possible 
accurately estimate the ``average expertise'' of the workers, 
captured by $\|\v\|$.
In order to avoid overcomplicating the model and the results and to 
circumvent the unidentifiability issue, we assume that $\|\v\|$ is known in our theoretical analysis.
In our experiments, we adopt the known-questions strategy to estimate $\|\v\|$.

\section{Spectral Estimators}
\label{sec:analysis}

Consider the general rank-one model described in Section~\ref{sec:model}.
The $n \times m$ binary matrix of observations $X$ can be written as
$
X = \E [X] + W,
$
where $\E[X] = \uvec \vvec^T$.
We assume throughout this section that $u_i, v_j \in [c,C]$,
where $0 < c < C < 1$. 
Moreover, we assume that $\|\vvec\|$ is known, in order to make the model identifiable, as described in Section~\ref{sec:model}.

A natural estimator for $\uvec$ is the leading left singular vector of $X$ (or the leading eigenvector of $XX^T$), rescaled to have the same norm as $\uvec$.
One issue with such an estimator is that it is not straightforward to obtain confidence intervals for each of its entries. 
There is a fair amount of work
in constructing $\ell_2$ confidence intervals around eigenvectors of perturbed matrices 
\citep{coja2010graph, chin2015stochastic, chen2016community}. 
However, 
the translation of $\ell_2$ control over eigenvectors to element-wise 
control using standard bounds costs us a $\sqrt{n}$ factor,
which makes the resulting bounds too loose for the purposes of adaptive algorithms, which we will explore in Section~\ref{sec:adaptive}.
There has been some work on directly obtaining $\ell_{\infty}$ control 
for eigenvectors by
\citet{abbe2017entrywise, fan2016ell}. 
However, they analyse a slightly different quantity, and so a direct application of these results 
to our rank-one models does not give us the desired element-wise control.

In order to overcome this issue
and obtain element-wise confidence bounds on our estimate of each $u_i$,
we propose a variation on the standard spectral estimator for $\uvec$.
To provide intuition to our method, let us consider a simpler setting -- one where
we know $\vvec$ exactly. 
In this case a popular means to estimate $\uvec$ is the \emph{matched 
filter} estimator \citep{verdu1998multiuser} 
\begin{align}
    \hat{u}_i &= X_{i,.}\frac{\vvec}{\|\vvec\|^2},
\end{align}
where $X_{i,.}$ is the $i$th row of $X$.
It is easy to see that $\hat{u}_i$ is an unbiased estimator of $u_i$, 
and standard concentration inequalities can be used to obtain confidence intervals.
We try to mimic 
this intuition by splitting the rows of the matrix into two -- red rows
and blue rows. 
We then use the red rows to obtain an estimate $\hat \vvec$ of $\v$. 
We treat 
this as the true value of $\v$ and obtain the matched filter estimate for the
$u_i$s corresponding to the blue
rows, which gives us element-wise
confidence intervals.
We then use the blue rows to estimate $\v$, and apply the matched filter to obtain estimates for the $u_i$s corresponding to the red rows.
This is summarised in the following algorithm.
\begin{algorithm}[H]
\begin{algorithmic}[1]
\caption{Spectral estimation of $\u$ 
\label{alg:spectral}}
\State \textbf{Input:} $X \in \{0,1\}^{n \times m}$, $\| \v \|$
\State Split $X$ into two $\tfrac{n}2 \times m$ matrices $X_A$ and $X_B$
\State $\hat \vvec_A \gets $ leading right singular vector of $X_A$ 
\State $\hat \vvec_B \gets $ leading right singular vector of $X_B$
\State $\hat \uvec_A \gets X_A \frac{\hat \vvec_B}{\|\hat \vvec_B\|\|\v\|}$,\  $\hat \uvec_B \gets X_B \frac{\hat \vvec_A}{\|\hat \vvec_A\|\|\v\|}$
\State \textbf{return} $\hat \uvec = \begin{bmatrix} \hat{\uvec}_A \\ \hat{\uvec}_B \end{bmatrix}$
\end{algorithmic}
\end{algorithm}
The main result in this section is an element-wise confidence interval for the resulting $\hat{\u}$. 
\begin{thm} \label{thm:final_confidence}
When given $X$ and $\|\v\|$ as inputs, Algorithm~\ref{alg:spectral} returns $\hat{\u} = [\hat{u}_1,...,\hat{u}_m]^T$ satisfying 
\begin{align} \label{eq:uibound}
    \Pr\left(|\hat{u}_i - u_i| > \epsilon\right) 
    &\le 3 n \exp \left( -C_1 { m \epsilon^2} \right),
\end{align}
for $i \in \{1,...,n\}$, $0< \ep < 1$, $m \leq n$, and constant $C_1$ specified in Appendix~\ref{app:constants}.
\end{thm}

In the remainder of this section, we describe the key technical results required to prove Theorem \ref{thm:final_confidence}.
We discuss the application of these confidence intervals to create an adaptive algorithm in Section \ref{sec:adaptive}.

To prove Theorem~\ref{thm:final_confidence} we 
first establish a connection between $\ell_2$ control of $\hat{\v}$ and element-wise control of $\hat{\u}$.
Then we provide expectation and tail bounds for the $\ell_2$ error in $\hat{\v}$.
For ease of exposition, we will drop the subscripts $A$ and $B$ in
$X_A$, $X_B$, $\hat{\v}_A$, and $\hat{\v}_B$.
We will implicitly assume that $X$ and $\hat{\v}$ correspond to distinct halves of the data matrix, thus being independent.
The main technical ingredient required to establish Theorem~\ref{thm:final_confidence} is the following lemma.

\begin{lem} \label{lem:elementwise}
The error of estimator $\hat{\u}$ satisfies
\aln{
\| \hat{\u} - \u\|_\infty \leq
\frac{1}{\|\v\|} \left|  \sum_{j=1}^m 
\left( X_{i,j} - u_iv_j \right) \E\left[ \frac{\hat{v}_j}{\|\hat{\v}\|} \right] \right|
+ 
\frac{1}{c} 
\left\|  \frac{\hat{\v}}{\|\hat{\v}\|} - \frac{\v}{\|{\v}\|}  \right\|
+ 
\frac{2}{c} \,
\E \left\|  \frac{\hat{\v}}{\|\hat{\v}\|} - \frac{\v}{\|{\v}\|}  \right\|.
}
\end{lem}
Notice that the right-hand side of the bound in Lemma~\ref{lem:elementwise} (proved in Appendix~\ref{app:elementwise})
involves the $\ell_2$ error of $\hat \v$, and can in turn be used to bound the $\ell_\infty$ error of the estimator $\hat \u$.
The first term in the bound is a sum of independent, bounded, zero-mean random variables $(X_{i,j}-u_iv_j) \E[\hat{v}_j/\|\hat{\v}\|]$, for $j=1,...,m$.
Using Hoeffding's inequality, we show in Appendix~\ref{app:hoeffding} that, for any $\ep > 0$,
\al{ \label{eq:hoeffding}
\P\left( \left| \textstyle{\sum_{j=1}^m} (X_{i,j}-u_iv_j) 
\E[\hat{v}_j/\|\hat{\v}\|]
\right| > \|\v\| \ep \right)  \leq 2 \exp\left( -{ 2 c^2 m \ep^2}\right).
}
In order to bound the second and third terms on the right-hand side of Lemma~\ref{lem:elementwise},
we resort to matrix concentration inequalities and the Davis-Kahan theorem \citep{davis1970rotation}.
More precisely, we have the following lemma, which we prove in Appendix~\ref{app:davis-kahan-long}.
\begin{lem} \label{lem:daviskahan}
The error of estimator $\hat{\v}$ satisfies
\begin{enumerate}[(a)]
\item $    \Pr\left( \left\| \frac{\hat{\v}}{\|\hat{\v}\|} - \frac{\v}{\|\v\|} \right\|  \ge \epsilon \right) \le 
    (m+n) \exp \left( -C_2 \ep^2 \min(m,n) \right)$, \quad for $0< \ep < 1$,
\item $\E \left\| \frac{\hat{\v}}{\|\hat{\v}\|} - \frac{\v}{\|\v\|} \right\| \leq
C_3 \sqrt{\frac{\log(m+n)}{\min(m,n)}}$,
\end{enumerate}
where $C_2$ and $C_3$ are constants specified in Appendix 
\ref{app:constants}.
\end{lem}

Given (\ref{eq:hoeffding}) and the bounds in Lemma~\ref{lem:daviskahan}, it is straightforward to establish Theorem~\ref{thm:final_confidence}, as we do next.
Fixing some $\ep \in (0,1)$, we have that if
$\ep^2 \min(m,n) > (6C_3/c)^2 \log(m+n)$, then
Lemma~\ref{lem:daviskahan}(b) implies that
\aln{
\frac{2}{c}\,
\E \left\| \frac{\hat{\v}}{\|\hat{\v}\|} - \frac{\v}{\|\v\|} \right\| \leq
\frac{2C_3}{c} \sqrt{\frac{\log(m+n)}{\min(m,n)}} < \frac{\ep}{3}.
}
Hence, if $\ep^2 \min(m,n) > (6 C_3/c)^2 \log(m+n)$,
\al{
\Pr\left(|\hat{u}_i - u_i| > \epsilon\right) & \leq 
\Pr\left( \left| \sum_{j=1}^m (X_{i,j}-u_iv_j) 
\E\left[ \frac{\hat{v}_j}{\|\hat{\v}\|} \right]  
\right| > \frac{\|\v\|\ep}{3} \right)
+ \Pr\left( \left\| \frac{\hat{\v}}{\|\hat{\v}\|} - \frac{\v}{\|\v\|} \right\|  \ge \frac{c \ep}3 \right) \nonumber \\
& \leq 2 \exp\left( -\frac{c^2 m \ep^2}{18}\right) + 
(m+n) \exp \left( -C_2 \frac{c^2 \epsilon^2  \min(m,n) }{9} \right)
\nonumber \\
& \leq (m+n+2) \exp \left( -C_4 \min(m,n) {\epsilon^2} \right). \label{eq:thm1proof}
}
Notice that \eqref{eq:thm1proof} is a vacuous statement whenever $C_4 \ep^2 \min(m,n) < \log(m+n)$, as the right-hand side is greater than $1$.
Hence, if we replace $C_4$ with $C_1 = \min[C_4,(6C_3/c)^{-2}]$, the inequality holds for all $m$ and $n$.
The result in Theorem~\ref{thm:final_confidence} then follows by assuming $m \leq n$.

\section{Leveraging confidence intervals for adaptivity}
\label{sec:adaptive}


In
the pairwise sequence alignment problem, one is typically only interested in identifying pairs of reads with large overlaps.
Hence, by discarding pairs of reads with a small overlap based on a coarse alignment estimate and adaptively picking the pairs of reads for which a more accurate estimate is needed, it is possible to save significant computational resources.
Similarly, in crowdsourcing applications, one may be interested in employing adaptive schemes in order to effectively use worker resources to identify only the most popular items.

We consider two natural problems that can be addressed within an adaptive framework.
The first one is the identification of the top-$k$ largest alignment scores.
In the second problem, the goal is to return a list of reads with high pairwise alignment to the reference, i.e., all reads with $u_i$ above a certain threshold.
More generally, we consider the task of identifying a set of reads including \emph{all} reads with pairwise alignment above $\alpha$ and no reads with pairwise alignment below $\beta$, for some $\beta \leq \alpha$.
Adaptivity in the 
first problem can be achieved by casting the problem as a top-$k$ multi-armed bandit problem,
while the second problem can be cast as a 
Thresholding Bandit problem \citep{locatelli2016optimal}.

\subsection{ Identifying the top-$k$ alignments: }
We consider the setting where we wish to find the $k$ largest pairwise alignments with a given read.
We assume that we have a total computational budget of $T$ min-hash comparisons.
Notice that the regime where $T < n$ is uninteresting as we cannot even make one min-hash comparison per read.
When $T = \Omega(n)$, a simple non-adaptive approach is to divide our budget $T$ evenly among all reads (as done in \citep{baharav2019spectral}). 
This gives us an $n \times \frac{T}{n}$ min-hash collision matrix $X$, from which we can estimate $\hat \u$ using Algorithm~\ref{alg:spectral} and choose the top $k$ alignments based on their $\hat u_i$ values.
Let $u_{(1)} \ge u_{(2)} \ge ... \ge u_{(n)}$ be
the sorted entries of the true $\uvec$ and define $\Delta_i = u_{(i)} - u_{(k)}$ for $i=1,...,n$.
Notice that the non-adaptive approach recovers $u_{(1)},...,u_{(k)}$ correctly if
each $\hat{u}_i$ stays within $\Delta_{k+1}$ of its true value $u_i$.
From the union bound and Theorem~\ref{thm:final_confidence}, 
\begin{align}
    \P(\text{failure}) &\le \sum_{i=1}^n \P\left(|\hat{u}_{i} - u_i| > \Delta_{k+1}/2\right) 
    \le 3n^2 \exp \left(- \frac{C_{1}}{4}\frac{\Delta_{k+1}^2 T}{n}  \right),
    \label{eq:nonadaptiveprob}
\end{align}
if $n \le T \leq n^2$.
Hence, the budget required to achieve an error probability of $\delta$ is
\al{
T=O\left( {\Delta_{k+1}^{-2}} n  \log \left( \tfrac{n}{\delta} \right)\right).
\label{eq:nonadaptiveT}
}
Moreover, from \eqref{eq:nonadaptiveprob}, we see that a budget $T = n\log^\beta n$, $\beta > 1$, allows you to correctly identify the top-$k$ alignments if
$(\Delta_{k+1}^2 T)/n \approx \log n$, or $\Delta_{k+1} \approx {\log^{-(\beta-1)/2}}n$.
Hence, the budget $T$ places a constraint in the minimum gap $\Delta_{k+1}$ that can be resolved.

Next we propose an adaptive way to allocate the same budget $T$.
Algorithm \ref{alg:bandit_topk} builds on an approach by \citet{Karnin2013}, but incorporates the spectral estimation approach from Section~\ref{sec:analysis}.
The algorithm assumes the regime $n \log n < T < n^2$.

\begin{algorithm}
\begin{algorithmic}[1]
\caption{Adaptive Spectral Top-$k$ Algorithm}\label{alg:bandit_topk}
\State \textbf{Input:} $T$, $k$ 
\State $\text{Initialize } \CI_0 \gets \{1,2,...,n\}$
\Comment{Initial set of candidates}
\For{$r=0$  \textbf{ to } $r_{\rm max} \triangleq \lceil \log_2 \frac{n}{\sqrt{T}.  } \rceil -1$}
\State 
$\displaystyle t_r \gets \left\lfloor \frac{T}{2 |\CI_r|\lceil \log_2 \frac{n}{ \sqrt{T}} \rceil} \right\rfloor$
\Comment{Number of samples to be taken}
\State Obtain a binary matrix $X^{(r)} \in \{0,1\}^{|\CI_r| \times t_r}$ and corresponding $\|\v^{(r)}\|$
\State Use Algorithm \ref{alg:spectral} to compute estimates
$\hat{\uvec}^{(r)}$ for $X^{(r)}$
\State{Set $\CI_{r+1}$ to be the $\lceil |\CI_r|/2 \rceil$ coordinates in $\CI_r$ with largest $\hat{\uvec}^{(r)}$}
\EndFor
\State{\textbf{Clean up:} Use $t_{r_{\max}+1} = \frac{T}{2}$, and compute $\hat{\uvec}^{(r_{\rm max}+1)}$ as above}
\State \textbf{return} the $k$ coordinates of $\CI_{r_{\rm max}+1}$ with the largest $\hat{\uvec}^{(r_{\rm max}+1)}$
\end{algorithmic}
\end{algorithm}

At the $r$-th iteration, Algorithm \ref{alg:bandit_topk} uses $t_r$ new hash functions and computes the min-hash collisions for the reads in $\CI_r$, which are represented by the matrix $X^{(r)}$.
Notice that we assume that the $t_r$ min-hashes in each iteration are different, which makes observation matrices $X^{(r)}$ all independent.
Also, we assume that the $\ell_2$ norm of the right singular vector of $\E X^{(r)}$, $\|\v^{(r)}\|$, can be obtained exactly at each iteration of the algorithm.
As discussed in Section~\ref{sec:model}, this makes the model identifiable and can be emulated in practice with calibration reads.
At each iteration, Algorithm \ref{alg:bandit_topk} eliminates half of the reads in $\CI_r$.
After $r_{\rm max}+1$ iterations, the number of remaining reads satisfies
$\sqrt{T}/2 \leq |\CI_{r_{\rm max}+1}| \leq 2\sqrt{T}$,
and the total budget used
is $\sum_{r=0}^{r_{\rm max}} |\CI_r| t_r \le T/2$.
Finally, in the ``clean up'' stage, we use the remaining $T/2$  budget to 
obtain the top $k$ among the approximately $\sqrt{T}$ remaining
items.
Notice that the final observation matrix is approximately $\sqrt{T} \times \sqrt{T}$.

In order to analyse the performance of Algorithm~\ref{alg:bandit_topk}, 
we proceed simiarly to
\citet{Karnin2013}, and define
$
    H_2 = \max_{i > k} \frac{i}{\Delta_i^2}.
$
We then have the following performance guarantee (proof in Appendix 
\ref{app:proof}).

\begin{thm} \label{thm:bandit_topk}
Given budget $n \log n\le T \le n^2$ and assuming $2k < \sqrt{T}$, 
Algorithm~\ref{alg:bandit_topk} correctly identifies the top $k$ alignments
with probability at least
\begin{align}
    &1 - 18kn \log n \exp \left(- \frac{C_1}{64} \frac{T  }{H_2\log n}\right) - 12n^2 \exp \left( -\frac{C_1}{16}\Delta_{k+1}^2 \sqrt{T}\right)
    \label{eq:adaptiveprob}
\end{align}
Moreover, for sufficiently small $\delta$, Algorithm~\ref{alg:bandit_topk} achieves an error probability of at most $\delta$ with budget
\al{
T=O\left( H_2 \log n \log \left( nk\frac{\log n}{\delta} \right)  + 
\Delta_{k+1}^{-4} \log^2 \left(\frac{n^2}{\delta} \right)
\right).
\label{eq:adaptiveT}
}
\end{thm}

Comparing \eqref{eq:adaptiveprob} and \eqref{eq:adaptiveT} with the non-adaptive counterparts \eqref{eq:nonadaptiveprob} and \eqref{eq:nonadaptiveT} requires a handle on $H_2$. 
This quantity captures how difficult it is to separate the top $k$ alignments, and satisfies 
$(k+1)\Delta_{k+1}^{-2}\le H_2\le n\Delta_{k+1}^{-2}$.
The extreme case $H_2 \approx n\Delta_{k+1}^{-2}$ occurs when all of the $n-k$ suboptimal items have very similar qualities.
In this case, adaptivity is not helpful.
However, when $u_{(k+1)}$ is large compared to other non-top-$k$ alignments, we are in the 
$H_2 \approx k \Delta_{k+1}^{-2}$ regime. 
Then the budget requirements are essentially $T = O(k \Delta_{k+1}^{-2} \log^2 n+ \Delta_{k+1}^{-4} \log^2 n)$, which is $O(\Delta_{k+1}^{-4} \log^2 n)$ for $k,\delta$ constant.
Furthermore, in the $H_2 \approx k \Delta_{k+1}^{-2}$ regime, a budget $T = n\log^\beta n$, $\beta > 2$, allows you to correctly identify the top-$k$ alignments with a gap of
$\Delta_{k+1} \approx n^{-1/4}$ which is  significantly  smaller than the $\Delta_{k+1} \approx {\log^{-(\beta-1)/2}}n$ afforded in the non-adaptive case. 
As a concrete example, suppose that out of the $n$ alignments, there are $k$ highly overlapping reads with $u_i = C$, $k$ moderately overlapping reads with $u_i = C-n^{-1/4}$, and $n-2k$ reads with no overlap and $u_i = c$.
In this case, Algorithm~\ref{alg:bandit_topk} requires a budget of $T = O \left(n\log^2 \left( \frac{n}{\delta}\right) \right)$, while the non-adaptive approach requires $T = O \left((n^{3/2} \log \left( \frac{n}{\delta}\right)\right)$.

\subsection{Identifying all alignments above a threshold:}
In this section, we develop a bandit algorithm to return a
set of coordinates in $\{1,...,n\}$ such that with high probability all coordinates with $u_i \ge \beta$
are returned and no coordinates with $u_i \le \alpha$
are returned, for some $\beta > \alpha$. 
We assume that  $\beta-\alpha > \sqrt{\frac{12 \log n}{C_4 n}}$.
The algorithm
and analysis follow similarly to \citet{locatelli2016optimal}.

\begin{lem}\label{lem:CI}
For any $\Gamma \in \R$, our estimates $\hat{u}_i$ from Algorithm \ref{alg:spectral} run with an $n\times m$ matrix with $m, n$ such that
\begin{align}
    \min(m, n) &\ge \frac{\log\left(\frac{1}{\delta}\right) + \log (m+n+2)}{\Gamma^2C_4},\label{eq:CI}
\end{align}
will have $|\hat{u}_i - u_i| \le \Gamma$ with probability 
at least $1 - \delta$.
Thus, if 
\begin{align}
    \min(m, n) &\ge  \frac{3\log n + 2\log (m+n) }{\Gamma^2C_4}
    \label{eq:CI2}
\end{align}
then with probability $1- \frac{1}{n^2}$, $|\hat{u}_i - u_i| \le \Gamma$
for all $i \in \{1,..., n\}$.
\end{lem}
\begin{proof}
Eq.~\eqref{eq:CI} follows by inverting the
error bound from Eq.~\eqref{eq:thm1proof}. Eq.~\eqref{eq:CI2} follows by just substituting $\delta = \frac{1}{n^3}$ and noting that $\log (m+n+2) \le 2 \log (m+n)$ as long as $m+n \ge 4$.
\end{proof}

Notice that a naive non-adaptive approach to this thresholding problem would consist of applying Algorithm~\ref{alg:spectral} to the entire observation matrix, with enough enough workers such that the confidence intervals are smaller than $\frac{\beta-\alpha}{2}$
and return coordinates with value more than $\frac{\alpha+\beta}{2}$.
In that case, Lemma \ref{lem:CI} implies that $\frac{12}{C_4}\frac{n\log n}{(\beta-\alpha)^2}$ workers' responses are enough to succeed with probability at least $1- \frac{1}{n}$.
In Algorithm~\ref{alg:thresh}, we propose an adaptive way of performing the same task.

\begin{algorithm}[H]
\begin{algorithmic}[1]
\caption{Adaptive Spectral Thresholding Algorithm}\label{alg:thresh}
\State \textbf{Input:} Range $[\alpha,\beta]$.
\State $\text{Initialise } S_0 \gets \{1,...,n\}$
\Comment{Set of coordinates initially under consideration}
\State $\text{Initialise } A \gets \emptyset$ 
\Comment{Set of coordinates initially accepted}
\State{$t_{-1} \leftarrow \frac{12\log n}{C_4}$}
\Comment{Initial number of workers recruited. $C_4$ from Eq. \eqref{eq:CI}}
\For{$r=0$ \textbf{to} $\lceil \log_2 \frac{1}{\beta-\alpha} \rceil -1$}
\State{$t_r = 4 t_{r-1}$} 
\Comment{Number of workers to be recruited}
\State Obtain a binary response matrix $X^{(r)} \in \{0,1\}^{|S_r| \times t_r}$ and corresponding $\|\v^{(r)}\|$
\State Use Algorithm \ref{alg:spectral} to compute estimates
$\hat{\uvec}^{(r)}$ for $X^{(r)}$
\State{Construct confidence intervals $C(t_r)$ (same for every question)}
\State{Construct accepted set $C_{\rm acc} = \{i \in S_r : \hat{u}^{(r)}_i - C(t_r) > \alpha\}$}
\State{Construct rejected set $C_{\rm rej} = \{i \in S_r :  \hat{u}^{(r)}_i + C(t_r) < \beta\} $} 
\State{Set $S_{r+1} = S_r \setminus \left\{C_{\rm rej} \cup C_{\rm acc} \right\}$}
\If{$|S_{r+1}| < \frac{12}{(\beta-\alpha)^2 C_{4}} \log n $}
\State {Let $\CI$ be $\frac{12}{(\beta-\alpha)^2 C_{4}} \log n - |S_{r+1}|$  coordinates of $C_{\rm rej} \cup C_{\rm acc}$ picked uniformly at random. \label{algLine:uarSet}}
\State {$S_{C} = S_{r+1} \cup \CI$}
\State {$A = A\cup (C_{\rm acc})$}
\State \textbf{break}
\Else
\State{Set $A = A\cup C_{\rm acc}$}
\EndIf
\EndFor
\State \textbf{Clean up:} Use $t_{C
} = \frac{12\log n}{(\beta-\alpha)^2C_4}$ workers with set $S_C$
of questions to obtain estimates $\hat{\uvec}^{(C)}$. Construct 
confidence intervals $C(t_C)$ and accepted set $C_A =\{ i \in S_C : \hat{u}^{(C)}_i - C(t_C) > \alpha\} $
\State {\textbf{return} $A \cup C_A$.}
\end{algorithmic}
\end{algorithm}

\begin{note}
Line 14 is used to make sure that we have $n>m$ in the clean up stage. Selecting $\CI$ uniformly at random from $C_{\rm rej} \cup C_{\rm acc}$ is simply given as a concrete way for the algorithm to run.
\end{note}

Let $\kappa := \left\lfloor\frac{12}{(\beta-\alpha)^2 C_{4}} \log n\right\rfloor$. 
Define 
\begin{align}
    \Gamma_i &= \begin{cases}
    u_i-\alpha & \text{ if } \beta < u_i,\\
    \beta-\alpha & \text{ if } \alpha \le u_i \le \beta, \\
    \beta - u_i & \text{ if } u_i < \alpha.
    \end{cases}\label{eq:Gamma}
\end{align}
See Figure \ref{fig:thresh_bandits} for an illustration.
Further let 
$\Gamma^{(1)}\le \Gamma^{(2)} \le \cdots \le \Gamma^{(n)}$ be
the sorted list of the $\Gamma_i$.
\begin{figure}[ht]
\vspace{-.5cm}
    \centering
    \includegraphics[width=.5\linewidth]{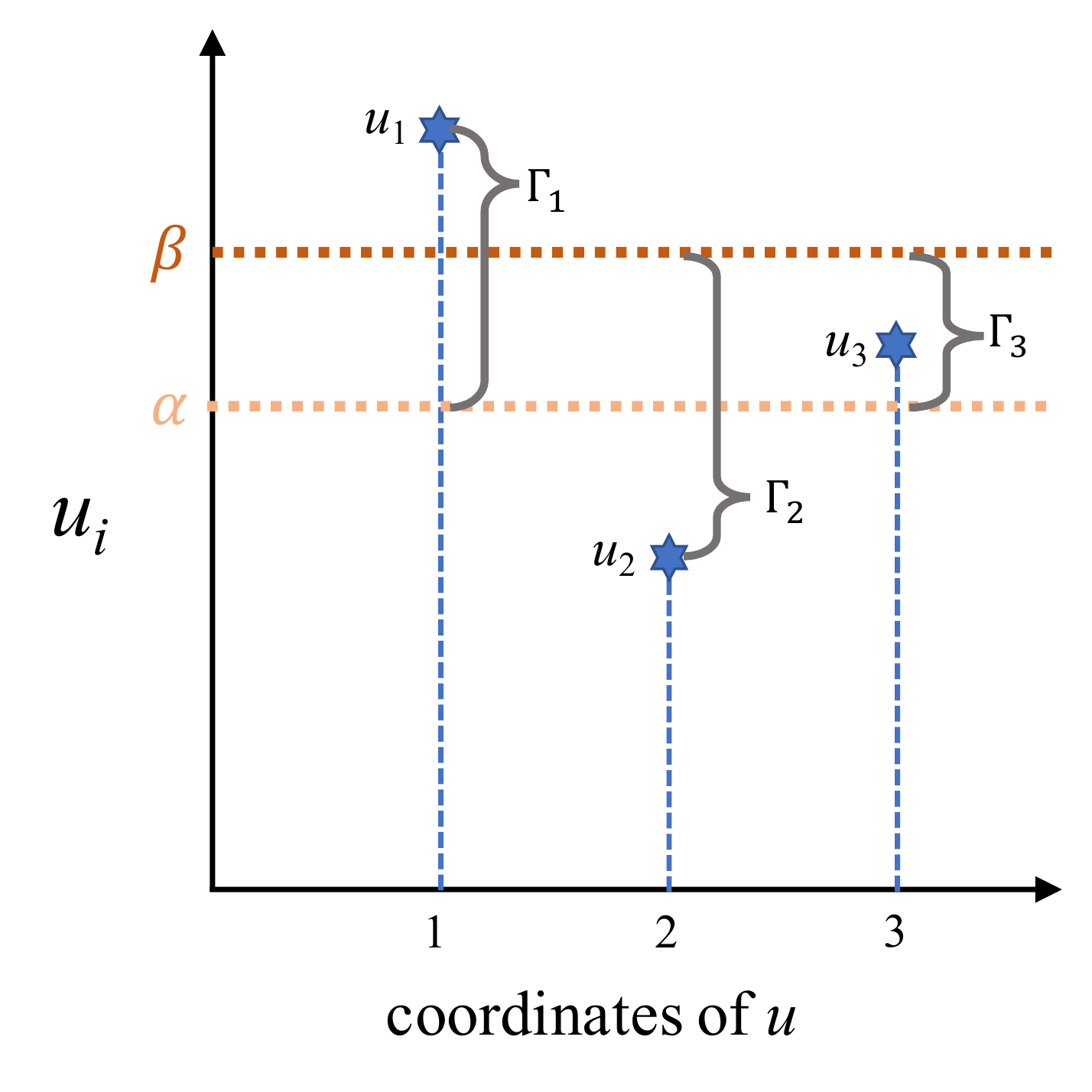}
    \vspace{-.2cm}
  \caption{An illustration of $\Gamma_i$ of Eq.~\eqref{eq:Gamma}}
\label{fig:thresh_bandits}
\vspace{-.5cm}
\end{figure}
\begin{thm}\label{thm:thresh_bdt_main}
Given parameters $\beta$ and $\alpha$ such that $\beta-\alpha > \sqrt{\frac{12 \log n}{C_4 n}}$,
with probability at least $1-\frac{2}{n}$
Algorithm~\ref{alg:thresh}
will output a set of reads $R$ such that 
$\{i: u_i > \beta\} \subseteq R \subseteq \{i: u_i > \alpha \}$ 
and use budget
\begin{align}
    T
    \le 2\left( \frac{12}{C_4}\frac{\log n}{(\beta-\alpha)^2}\right)^2 + \sum_{\ell=\kappa +1}^n \frac{32}{C_4}\frac{\log n}{(\Gamma^{(\ell)})^2}\label{eq:thresholdingtotal1}
\end{align}
 where $\Gamma_i$
denotes the difficulty of classifying $u_i$, with
$\Gamma^{(1)}\le \cdots \le \Gamma^{(n)}$ as
the sorted list of the $\Gamma_i$.
\end{thm}

The proof of the theorem is in Appendix \ref{app:threhsold}.

\begin{note}
Theorem \ref{thm:thresh_bdt_main} enables us to construct
confidence intervals of $\frac{\beta-\alpha}{2}$ for only the
``most borderline'' $\kappa$ questions and near optimal confidence 
intervals for the rest, while the non-adaptive algorithm would need to construct
confidence intervals of $\frac{\beta-\alpha}{2}$ for all.
\end{note}

\section{Empirical Results}\label{sec:results}

In order to validate the Adaptive Spectral Top-$k$ algorithm, 
we conducted two types of 
experiments:
\begin{enumerate}
    \item controlled experiments on simulated data for a 
crowdsourcing model with symmetric errors;
\item  pairwise sequence alignment experiments on real DNA sequencing data.
\end{enumerate}
We consider the top-$k$ identification problem with $k=5$ in 
both cases. 
We run Algorithm \ref{alg:bandit_topk} with some slight modifications, namely halving until we have fewer
than $2k$ remaining arms before moving to the clean up step, and
compare its performance with the non-adaptive spectral approach.
Further experimental details are in Appendix \ref{app:implementation}.
We measure success in two ways.
First, we consider the error probability of returning the top $k$ items (i.e., any deviation from the top-$k$ is considered a failure).
Second, we consider a less stringent metric, where we allow the algorithm to return its top-$2k$ items, and we consider the fraction of the true top-$k$ items that are present to evaluate performance.
Our code is publicly available online at \href{http://www.github.com/TavorB/adaptiveSpectral}{github.com/TavorB/adaptiveSpectral}.

\begin{figure}[ht] 
    \centering
    \includegraphics[width=.75\linewidth]{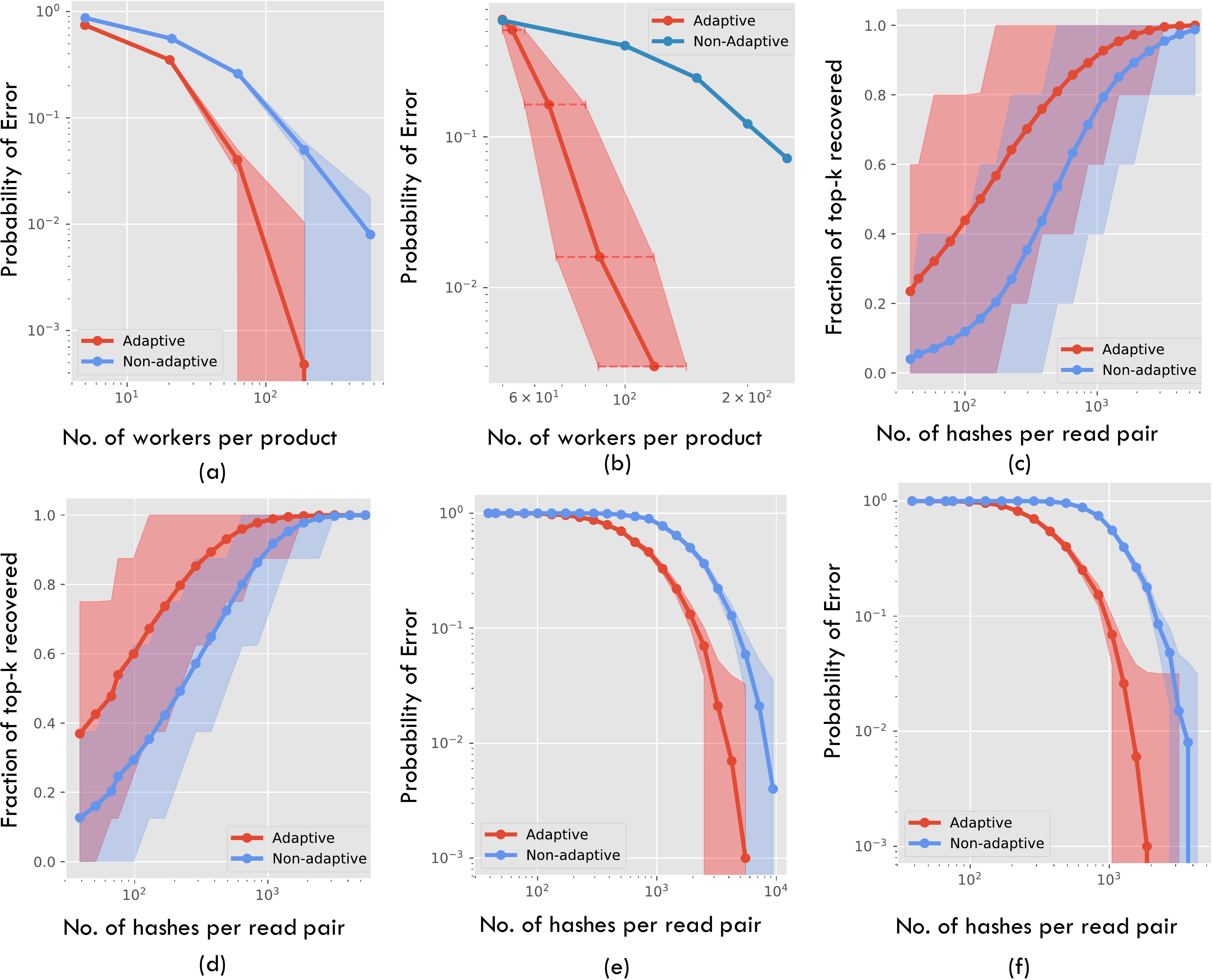} 
  \caption{
  (a) shows the  
probability of error of the
controlled 
crowdsourcing experiment as the number of workers per product is increased, where an error is defined as incorrectly identifying the set of top-$k$ products. (b)~shows the same in the thresholding bandits setup.
 (c) shows the fraction of 
the top-$k$ reads that are in the $2k$ reads returned for
the  \emph{E.~coli} dataset, while 
(d) shows the fraction of the top-$k$ reads that are in the top-$k$
reads returned for the NCTC4174 dataset.
(e)  shows the probability of error of correctly identifying the set of top-$k$ overlaps on the \emph{E. coli} dataset, while (f) shows the same
for the NCTC4174 dataset.
 1000 
trials are conducted for each point. $95\%$ percentiles are shaded around each point in (b), (c) and (d) (note that confidence intervals for (b) are on the x-axis). For (a), (e) and (f), $\frac{1}{\sqrt \text{number of trials}}$ is shaded around each point. For further details, see Appendix \ref{app:implementation}.}
\label{fig:results}
\vspace{-0.5cm}
\end{figure}

{\bf Controlled experiments:}
We consider a crowdsourcing scenario with symmetric errors
as modelled in \eqref{eq:symmetricmodel}.
We want to determine the $5$ best products from a list
of $1000$ products. 
We generate the true product qualities (that is, the $p_i$ parameters)
from a Beta$(1,5)$ distribution independent of each 
other. Each
of the worker abilities $q_j$ is drawn from a
Uniform$(0,1)$ distribution, independent of everything else. 
We consider the problem of top-$5$ 
product detection at various budgets
as shown in Figure~\ref{fig:results}(a) with success rate 
measured by the presence in the top-$10$ items. We see that
the adaptive
algorithm requires significantly fewer worker responses to achieve equal performance to the non-adaptive one. 

In Figure~\ref{fig:results}(b) we consider the same set up 
as above but in the fixed confidence setting, considering the problem
of being able to detect all products that are liked by more than $65\%$
of the population while returning none that is liked by less than $50\%$
of the population. Again, we see that for the same probability of error 
the adaptive algorithm needs far fewer workers.

\textbf{Real data experiments:}
Using the PacBio \emph{E.~coli} data set \citep{PBecoli} that
was examined in \citet{baharav2019spectral} we consider the problem of 
finding, for a fixed reference read, the $5$ reads that 
have the largest alignment with the reference read
in the dataset. We show the fraction of the $5$ reads that are present when returning $10$ reads
in Figure~\ref{fig:results}(c) and the success probability when returning exactly $5$
reads in Figure~\ref{fig:results}(e) (i.e., the probability of returning exactly the top-$5$ reads). 
To achieve an error rate of 0.9\% the non-adaptive algorithm requires over 8500 min-hash comparisons per read, while the adaptive algorithm requires fewer than 6000 per read to reach an error rate of 0.1\%.

We also consider the NCTC4174 dataset of \citep{NCTC3000} and plot the 
fraction correct when returning $5$ reads in  Figure~\ref{fig:results}(d) and the success probability when returning exactly $5$
reads in Figure~\ref{fig:results}(f). The results are qualitatively
similar to what we observe in the case of the \emph{E.~coli} dataset.

\section{Discussion}
Motivated by applications in sequence 
alignment, we considered the problem of efficiently finding the largest elements in the left singular vector of a binary matrix $X$ with $\mathbb{E}[X] = \mathbf{u}\mathbf{v}^\top$.
To utilize the natural spectral 
estimator of $\mathbf{u}$, we designed a method to construct
$\ell_\infty$ confidence intervals around the spectral estimator.
To perform this spectral estimation efficiently, we 
leveraged multi-armed bandit algorithms to adaptively estimate the entries $u_i$ of the leading left singular vector to the necessary degree of accuracy.
We show that this method provides computational gains on both
real data and in controlled experiments.

\clearpage

\section*{Broader Impact} 
Over the last decade, high-throughput sequencing technologies have driven down the time and cost of acquiring biological data tremendously. 
This has caused an explosion in the amount of available genomic data, allowing scientists to obtain quantitative insights into the biology of all living organisms.
Countless tasks -- such as gene expression quantification, metagenomic sequencing, and single-cell RNA sequencing -- heavily rely on some form of pairwise sequence alignment, which is a heavy computational burden and often the bottleneck of the analysis pipeline.
The development of efficient algorithms for this task, which is the main outcome of this paper, is thus critical for the scalability of genomic data analysis.

From a theoretical perspective, this work establishes novel connections between a classical problem in bioinformatics (pairwise sequence alignment), spectral methods for parameter estimation from crowdsourced noisy data, and multi-armed bandits.
This will help facilitate the transfer of insights and algorithms between these traditionally disparate areas.
It will also add a new set of techniques to the toolbox of the computational biology community that we believe will find a host of applications in the context of genomics and other large-scale omics data analysis.
Further, this work will allow other Machine Learning researchers unfamiliar with bioinformatics to utilise their expertise in solving new problems at this novel intersection of bioinformatics, spectral methods, and multi-armed bandits.

\begin{ack}
G.M.~Kamath would like to thank Lester Mackey of Microsoft Research, New England Lab for useful discussions on Rank-one models and their connection to crowd-sourcing. The research of T.~Baharav was supported in part by the Alcatel Lucent Stanford Graduate Fellowship and NSF GRFP.
The research of I.~Shomorony was supported in part by NSF grant CCF-2007597.


\end{ack}

\clearpage

\bibliography{mybib}
\bibliographystyle{abbrvnat}

\clearpage
\appendix
\appendixpage

\section{A Rank-$2$ Model}
\label{app:rank2}
Consider the case where we the $(i,j)$-th observation
$Y_{i,j}$ is the output of a Ber$(p_i)$ random variable passed through 
a general binary channel BC$(q^{(0)}_j, q^{(1)}_j)$ (see Figure~\ref{fig:channels}(c)). 
Here $q^{(0)}_j$ is probability of a $0$ being flipped to 
a $1$ and $q^{(1)}_j$ is probability of a $1$ being flipped to 
a $0$ on the $j$th column.

We note that we have $n+2m$ parameters here, while 
a rank-one model would admit only $n+m$ parameters. Hence this
is not a rank-one model.
However we note that
\begin{align}
    \E [Y] &= \p (\1-\q^{(0)})^T + (\1 - \p) \q^{(1)T},
\end{align}
where $\p = [p_1,...,p_n]^T$, $\q^{(0)}= [q^{(0)}_1,...,q^{(0)}_m]^T$ and $\q^{(1)}= [q^{(1)}_1,...,q^{(1)}_m]^T$. 
This shows that, when the noise in the workers' responses is modelled by a general binary channel (see Figure~\ref{fig:channels}(c)), we have a rank-$2$ model.

\section{Proof of Lemma~\ref{lem:elementwise}}
\label{app:elementwise}

We start by using the triangle inequality to obtain
\begin{align}
    |\hat{u}_i - u_i| \leq | \hat{u}_i - \E[\hat{u}_i] | + | \E[\hat{u}_i] - u_i |. \label{eq:triangle}
\end{align}

To bound the first term, we first notice that
since
\al{\hat u_i = X_{i,.} \frac{\hat{\v}}{\|\hat{\v}\| \|\v\|}
= \frac{1}{\|\v\|} \sum_{j=1}^m 
\left( X_{i,j} \frac{\hat{v}_j}{\|\hat{\v}\|} \right)
} 
and $\hat{\v}$ is independent of $X_{i,.}$, we have that 
\al{
\hat{u}_i - \E[\hat{u}_i]
& = \frac{1}{\|\v\|} \sum_{j=1}^m 
\left( X_{i,j} \frac{\hat{v}_j}{\|\hat{\v}\|} \right)
- \frac{1}{\|\v\|} \sum_{j=1}^m 
\left( u_iv_j \E\left[ \frac{\hat{v}_j}{\|\hat{\v}\|} \right] \right)
\nonumber \\
& = 
\frac{1}{\|\v\|} \sum_{j=1}^m 
\left( X_{i,j} \E\left[ \frac{\hat{v}_j}{\|\hat{\v}\|} \right]
+ X_{i,j} \left( \frac{\hat{v}_j}{\|\hat{\v}\|} - \E\left[ \frac{\hat{v}_j}{\|\hat{\v}\|} \right] \right) 
-u_iv_j \E\left[ \frac{\hat{v}_j}{\|\hat{\v}\|} \right] \right) \nonumber \\
& = 
\frac{1}{\|\v\|} \sum_{j=1}^m 
\left( X_{i,j} - u_iv_j \right) \E\left[ \frac{\hat{v}_j}{\|\hat{\v}\|} \right] 
+ 
\frac{1}{\|\v\|} \sum_{j=1}^m 
X_{i,j} \left( \frac{\hat{v}_j}{\|\hat{\v}\|} - \E\left[ \frac{\hat{v}_j}{\|\hat{\v}\|} \right] \right). \label{eq:app2}
}
From the triangle inequality, we then have 
\al{
|\hat{u}_i - \E[\hat{u}_i]|
& \leq 
\frac{1}{\|\v\|} \left|  \sum_{j=1}^m 
\left( X_{i,j} - u_iv_j \right) \E\left[ \frac{\hat{v}_j}{\|\hat{\v}\|} \right] \right|
+ 
\frac{1}{\|\v\|} \sum_{j=1}^m 
|X_{i,j} | \left|  \frac{\hat{v}_j}{\|\hat{\v}\|} - \E\left[ \frac{\hat{v}_j}{\|\hat{\v}\|} \right] \right| \nonumber \\
& \leq 
\frac{1}{\|\v\|} \left|  \sum_{j=1}^m 
\left( X_{i,j} - u_iv_j \right) \E\left[ \frac{\hat{v}_j}{\|\hat{\v}\|} \right] \right|
+ 
\frac{1}{\|\v\|} 
\left\|  \frac{\hat{\v}}{\|\hat{\v}\|} - \E\left[ \frac{\hat{\v}}{\|\hat{\v}\|} \right] \right\|_1 
\label{eq:app3}
}
From the fact that $\|\v\| \geq c\sqrt{m}$ and 
the fact that $\|\x \|_1 \leq \sqrt{m} \|\x\|_2$ for any $\x \in \R^m$,
the second term can be bounded as 
\al{
\frac{1}{\|\v\|} 
\left\|  \frac{\hat{\v}}{\|\hat{\v}\|} - \E\left[ \frac{\hat{\v}}{\|\hat{\v}\|} \right] \right\|_1 
& \leq 
\frac{1}{c} 
\left\|  \frac{\hat{\v}}{\|\hat{\v}\|} - \E\left[ \frac{\hat{\v}}{\|\hat{\v}\|} \right] \right\| \nonumber \\
& \leq \frac{1}{c} 
\left\|  \frac{\hat{\v}}{\|\hat{\v}\|} -  \frac{\v}{\|\v\|}  \right\|
+
\frac{1}{c} 
\left\|  \frac{\v}{\|\v\|} - \E\left[ \frac{\hat{\v}}{\|\hat{\v}\|} \right] \right\| \nonumber \\
& = \frac{1}{c} 
\left\|  \frac{\hat{\v}}{\|\hat{\v}\|} -  \frac{\v}{\|\v\|}  \right\|
+
\frac{1}{c} 
\left\|  \E \left[\frac{\v}{\|\v\|} - \frac{\hat{\v}}{\|\hat{\v}\|} \right] \right\| \nonumber \\
& \leq \frac{1}{c} 
\left\|  \frac{\hat{\v}}{\|\hat{\v}\|} -  \frac{\v}{\|\v\|}  \right\|
+
\frac{1}{c} 
\E \left\| \frac{\v}{\|\v\|} - \frac{\hat{\v}}{\|\hat{\v}\|} \right\|,
\label{eq:triangle2}
}
where the last step follows from Jensen's inequality. 

Now we consider the second term in \eqref{eq:triangle}.
We first notice that
\aln{
\hat u_i = X_{i,.} \frac{\hat{\v}}{\| \hat{\v} \| \|\v\|}
= \frac{X_{i,.}}{\|\v\|} \left( \frac{\v}{\|\v\|}  + \frac{\hat{\v}}{\|\hat{\v}\|} - \frac{\v}{\|\v\|} \right)
= \frac{X_{i,.} \v}{\|\v\|^2} + \frac{X_{i,.}}{\|\v\|} \left(  \frac{\hat{\v}}{\|\hat{\v}\|} - \frac{\v}{\|\v\|} \right),
}
where we recognize the first term as the matched filter for estimating $u_i$ if $\v$ were known.
Since $E[X_{i,.}] = u_i \v^T$ and $\hat{\v}$ is independent of $X_{i,.}$ (due to the splitting of the data matrix $X$),
we have
\aln{
\E[\hat u_i] 
= \frac{u_i \v^T \v}{\|\v\|^2} + \frac{u_i \v^T}{\|\v\|} \E \left[  \frac{\hat{\v}}{\|\hat{\v}\|} - \frac{\v}{\|\v\|} \right]
= u_i + \frac{u_i \v^T}{\|\v\|}  \left(  \E \left[ \frac{\hat{\v}}{\|\hat{\v}\|} \right] - \frac{\v}{\|\v\|} \right).
}
Using the Cauchy-Schwarz inequality, we have that
\al{
\left| \E[\hat u_i] - u_i \right| 
\leq \frac{u_i \|\v\|}{\|\v\|} \left\| \E\left[\frac{\hat{\v}}{\|\hat{\v}\|}\right] - \frac{\v}{\|\v\|} \right\|
= u_i \left\| \E\left[\frac{\hat{\v}}{\|\hat{\v}\|}\right] - \frac{\v}{\|\v\|} \right\|
\leq u_i \E \left\| \frac{\hat{\v}}{\|\hat{\v}\|}- \frac{\v}{\|\v\|} \right\|, \label{eq:triangle3}
}
where the last step follows from Jensen's inequality.

Finally, putting together \eqref{eq:triangle}, \eqref{eq:app3}, \eqref{eq:triangle2}, and \eqref{eq:triangle3}, and noting that $u_i < 1 < 1/c$, we obtain
\aln{
|\hat{u}_i - u_i| \leq 
\frac{1}{\|\v\|} \left|  \sum_{j=1}^m 
\left( X_{i,j} - u_i \right) E\left[ \frac{\hat{v}_j}{\|\hat{\v}\|} \right] \right|
+ 
\frac{1}{c} 
\left\|  \frac{\hat{\v}}{\|\hat{\v}\|} - \frac{\v}{\|{\v}\|}  \right\|
+ 
\frac{2}{c}\,
\E \left\|  \frac{\hat{\v}}{\|\hat{\v}\|} - \frac{\v}{\|{\v}\|}  \right\|.
}

\section{Proof of Equation \eqref{eq:hoeffding}} \label{app:hoeffding}

We claim that for any $\ep > 0$,
\al{ 
\P\left( \left| \textstyle{\sum_{j=1}^m} (X_{i,j}-u_iv_j) 
\E[\hat{v}_j/\|\hat{\v}\|]
\right| > \|\v\| \ep \right)  \leq 2 \exp\left( -{ 2 c^2 m \ep^2}\right).
}
First we notice that the random variables $(X_{i,j}-u_iv_j) \E\left[ \frac{\hat{v}_j}{\|\hat{\v}\|} \right] $, for $j=1,...,m$, are independent and zero-mean.
Moreover, they satisfy 
\aln{
-u_iv_j \E\left[ \frac{\hat{v}_j}{\|\hat{\v}\|} \right]  < (X_{i,j}-u_iv_j) \E\left[ \frac{\hat{v}_j}{\|\hat{\v}\|} \right] < (1-u_iv_j)\E\left[ \frac{\hat{v}_j}{\|\hat{\v}\|} \right].
}
Using Hoeffding's inequality, for any $\ep > 0$ we have that
\aln{ 
\P\left( \left| \sum_{j=1}^m (X_{i,j}-u_iv_j) 
\E\left[ \frac{\hat{v}_j}{\|\hat{\v}\|} \right]  
\right| > \|\v\| \ep \right) & \leq 2 \exp\left( -\frac{ 2 \|\v \|^2 \ep^2}{\sum_{j=1}^m \E[\hat{v}_j/\|\hat{\v}\|]^2 }\right) \\
& \leq 2 \exp\left( -\frac{ 2 c^2 m \ep^2}{\E\left[\sum_{j=1}^m \hat{v}_j^2/\|\hat{\v}\|^2\right] }\right) \\
& = 2 \exp\left( -{ 2 c^2 m \ep^2}\right),
}
where in the second step we used Jensen's inequality.

\section{Proof of Lemma~\ref{lem:daviskahan}} \label{app:davis-kahan-long}

Let $W = X - \E X$ be the ``noise'' added to $\E X$.
In order to prove Lemma~\ref{lem:daviskahan},
our first order of business is to bound the expectation of $\|W\|_{op}$ and $\|W\|_{op}^2$.
Then we use these bounds with the Davis-Kahan theorem
of \citet{davis1970rotation} to bound the $\ell_2$ error in $\hat{\v}$. 
We have the following lemma.

\begin{lem}\label{lem:Wop2}
The noise matrix $W = X - \E X$ satisfies 
\al{
& \E \left[ \| W \|_{op} \right] \leq 2\sqrt{(m+n) \log (m + n)} \label{eq:Woplem1}\\
& \E \left[\| W \|^2_{op}\right] 
\leq 120 (m+n) \log (m + n)
\label{eq:Woplem2}
}
\end{lem}

\begin{proof}
Strictly speaking, due to Jensen's inequality, \eqref{eq:Woplem2} implies \eqref{eq:Woplem1} (with a different constant). 
However, to provide intuition and improve the exposition, we provide a standalone proof of \eqref{eq:Woplem1} first. 
We start by noticing that
\begin{align*}
    W_{i,j} &= \begin{cases}
    1-u_iv_j & \text{ with probability } u_i v_j,\\
    -u_iv_j & \text{ with probability } 1-u_i v_j,
                \end{cases}
\end{align*}
Hence we have
\begin{align*}
    (\mathbb{E}[W^TW])_{i,j} &=  \begin{cases}
    0 & \text{ if } i \ne j,\\
    \sum_{k=1}^n u_kv_i (1-u_kv_i) & \text{ if } i=j,
                \end{cases}
\end{align*}
which implies that  $c^2(1-C^2)n \le (\mathbb{E}[W^TW])_{i,i} \le C^2(1-c^2)n$.
Thus $\|\mathbb{E}[W^TW] \|_{op} \in [c^2(1-C^2)n, C^2(1-c^2)n]$. Similarly one can argue that 
$\|\mathbb{E}[WW^T] \|_{op} \in [c^2(1-C^2)m, C^2(1-c^2)m]$.
Following the notation of \citet[Theorem 6.1.1]{tropp2015introduction}, 
the matrix variance statistic $\nu(W)$ is 
\begin{align}
    \nu(W) &= \max\left(\|\mathbb{E}[W^TW] \|_{op} , \|\mathbb{E}[WW^T] \|_{op}\right),\nonumber\\
    &\in [c^2(1-C^2) \max(m,n), C^2(1-c^2)\max(m,n)],\nonumber\\
    &\le  C^2(1-c^2) (m+n) \leq m + n.\label{eq:mat_var}
\end{align}
From the Matrix-Bernstein inequality \citep[Eq. (6.1.3)]{tropp2015introduction}, we 
have that
\aln{
\E \|W\|_{op} & \leq \sqrt{2 \nu(W) \log (m + n) } + \tfrac{1}{3} \log(m+n) \\
& = \sqrt{2 (m+n) \log (m + n) } + \tfrac{1}{3} \log(m+n) \\ 
& \leq (\sqrt2 + \tfrac13) \sqrt{(m+n) \log (m + n) } \leq 2\sqrt{(m+n) \log (m + n) },
}
proving \eqref{eq:Woplem1}. 
To prove \eqref{eq:Woplem2}, we rely on another inequality by \citet[Eq. (6.1.6)]{tropp2015introduction} to state that
\al{
\left(\E \|W\|_{op}^2\right)^{1/2} & \leq \sqrt{2 e \nu(W) \log (m + n) } + 4e \log(m+n)  \nonumber \\& \leq \sqrt{2 e (m+n) \log (m + n) } + 4e \log(m+n) \nonumber \\
& \leq 4e \sqrt{ (m+n) \log (m + n) }, \nonumber 
}
which implies that
$\E \|W\|_{op}^2 \leq 120 (m+n) \log (m + n)$,
proving \eqref{eq:Woplem2}. 
\end{proof}

With Lemma~\ref{lem:Wop2}, we proceed to the proof of Lemma~\ref{lem:daviskahan}.
Notice that the leading right singular vector of $X$
is equivalent to the leading eigenvector of $X^TX$.
Also note that
\begin{align}
    X^TX &= (\mathbb{E}X + W)^T(\mathbb{E}X + W) \nonumber\\
    &=  (\u\v^T + W)^T(\u\v^T + W) \nonumber\\
    &= \|\u\|^2 \v \v^T + \v\u^T W + W^T \u\v^T + W^TW.
    \label{eq:app1}
\end{align}
We will use the Davis-Kahan theorem to bound 
$\left\| \frac{\hat{\v}}{\|\hat{\v}\|}- \frac{\v}{\|\v\|} \right\|$. We begin by bounding the operator
norm of the ``error terms'' in \eqref{eq:app1} as
\begin{align}
    \|\v\u^T W + W^T \u\v^T + W^TW\|_{op} &\overset{(a)}{\le }\|\v\u^T W\|_{op} + \|W \u\v^T\|_{op} + \|W^TW\|_{op}, \nonumber \\
    &\overset{(b)}{\le }\|\v\u^T\|_{op} \|W\|_{op} + \|W\|_{op} \|\u\v^T\|_{op} + \|W^TW\|_{op}, \nonumber\\
    &= 2 \|\u\|\|\v\| \|W\|_{op} + \|W^TW\|_{op}, \label{eq:err_norm1}
\end{align}
where $(a)$ follows from the triangle inequality, and $(b)$ 
from the fact that $\|AB\|_{op} \le \|A\|_{op} \|B\|_{op}$.
We also note that
\begin{align}
    c \sqrt{n} &\le \|\u\| \le C \sqrt{n} \le \sqrt{n},\\
    c \sqrt{m} &\le \|\v\| \le C \sqrt{m} \le \sqrt{m}.
\end{align}
Since $\|\u \|^2 \v \v^T$ is rank-one with leading
eigenvalue $\|\u\|^2\|\v\|^2$, 
the spectral gap $\delta$
of $\|\u \|^2 \v \v^T$ is $\delta = \|\u\|^2\|\v\|^2 \ge c^4(mn)$. 
From the version of the Davis-Kahan
Theorem \citep{davis1970rotation} of \citet[Theorem 30]{mahoney2016lecture}, we have that
\begin{align}
    \left\| \frac{\hat{\v}}{\|\hat{\v}\|}- \frac{\v}{\|\v\|} \right\| &\le \sqrt{2} \frac{\|\v\u^T W + W^T \u\v^T + W^TW\|_{op}}{\delta} \nonumber \\
    & \leq \frac{\sqrt{8mn}\| W\|_{op} + \sqrt{2}\| W^TW\|_{op}}{c^4 (mn)} = \frac{\sqrt{8} \| W\|_{op}}{c^4\sqrt{mn}} + \frac{\sqrt{2} \| W\|^2_{op}}{c^4 mn}.
    \label{eq:davis-kahan}
\end{align}
Taking the expectation on both sides and using Lemma~\ref{lem:Wop2}, we obtain
\al{
    \E \left\| \frac{\hat{\v}}{\|\hat{\v}\|}- \frac{\v}{\|\v\|} \right\| &
    \leq \frac{\sqrt{16(m+n)\log (m+n)}}{c^4\sqrt{mn}} + 
    \frac{120 \sqrt{2} (m+n) \log (m + n)}{c^4 mn} \nonumber
    \\
    & = \frac{4}{c^4} \sqrt{\left(\frac{1}{m}+\frac{1}{n}\right)\log(m+n)}
    + \frac{120\sqrt{2}}{c^4} \left(\frac{1}{m}+\frac{1}{n}\right)\log(m+n) 
    \nonumber \\
    & 
    \leq \frac{4}{c^4} \sqrt{\frac{\log(m+n)}{\min(m,n)}}
    + \frac{120\sqrt{2}}{c^4} \frac{\log(m+n)}{\min(m,n)}
    \label{eq:appbias}
}
Notice that $\hat{\v}/\|\hat{\v}\|$ and ${\v}/\|{\v}\|$ are unit vectors, and so their $\ell_2$ distance can be at most $2$.
The right-hand side of \eqref{eq:appbias} can only be less than $2$ if 
\aln{
\frac{4}{c^4} \sqrt{\frac{\log(m+n)}{\min(m,n)}} \leq 1
\; \Rightarrow \;
\frac{\log(m+n)}{\min(m,n)} \leq \frac{c^4}{4}\sqrt{\frac{\log(m+n)}{\min(m,n)}}.
}
Hence, 
\al{
    \E \left\| \frac{\hat{\v}}{\|\hat{\v}\|}- \frac{\v}{\|\v\|} \right\| 
    &
    \leq 
    \min\left[2,\frac{4}{c^4} \sqrt{\frac{\log(m+n)}{\min(m,n)}}
    + \frac{120\sqrt{2}}{c^4} \frac{\log(m+n)}{\min(m,n)}\right] \nonumber \\
    &
    \leq \left(\frac{4}{c^4} + 30\sqrt{2}\right) \sqrt{\frac{\log(m+n)}{\min(m,n)}}.
}
This proves Statement (b) in Lemma~\ref{lem:daviskahan}, where we can take $C_3 = 4/c^4 + 30\sqrt{2}$.
To prove Statement (a), we note that
from \eqref{eq:davis-kahan},
\begin{align}
    \left\| \frac{\hat{\v}}{\|\hat{\v}\|}- \frac{\v}{\|\v\|} \right\| 
    & \leq \frac{\sqrt{8} \| W\|_{op}}{\|\u\| \|\v\|} + \frac{\sqrt{2} \| W\|^2_{op}}{\|\u\|^2 \|\v\|^2} 
\end{align}
Next we notice that, if $\frac{\| W\|_{op}}{\|\u\| \|\v\|} <  \ep/4$, then
\aln{
\left\| \frac{\hat{\v}}{\|\hat{\v}\|}- \frac{\v}{\|\v\|} \right\| 
    & < \frac{\ep}{\sqrt2} + \frac{\ep^2}{8\sqrt{2}} < \ep,
}
for $0 < \ep < 1$.
Therefore, we have that
\al{
\P\left( \left\| \frac{\hat{\v}}{\|\hat{\v}\|} - \frac{\v}{\|\v\|} \right\|  \ge \epsilon \right) 
& \le  
\P\left( \| W\|_{op} >  \tfrac14 \ep \|\u\| \|\v\| \right)
\nonumber \\
& \le  
\P\left( \| W\|_{op} >  \tfrac14 \ep c^2 \sqrt{mn} \right)
\nonumber \\
& \leq 
(m+n) \exp \left( -\frac{1}{32}\frac{c^4\epsilon^2 mn}{m+n + \tfrac1{12} c^2 \ep \sqrt{mn} } \right) \label{eq:step1} \\
& \leq 
(m+n) \exp \left( -\frac{1}{32}\frac{c^4\epsilon^2 mn}{(2+\tfrac1{12}) \max(m,n)} \right) \label{eq:step2}  \\
& \leq 
(m+n) \exp \left( -\frac{c^4\epsilon^2}{48}\min(m,n) \right) \nonumber
}
where \eqref{eq:step1} follows by the Matrix-Bernstein 
inequality \citep[Eq. (6.1.4)]{tropp2015introduction} using the computation of the matrix variance 
statistic from \eqref{eq:mat_var}, and \eqref{eq:step2} follows
since $c < 1$ and $\ep < 1$.
This means we can take $C_2 = c^4/48$.

\section{Proof of Theorem \ref{thm:bandit_topk}} \label{app:proof}

The proof at a high level proceeds by showing that the probability we eliminate any of the top $k$ arms is low in our halving stages, and then that uniformly sampling the $\sqrt{T}$ remaining arms allows us to identify the top $k$ correctly. Note that coordinates, items, and arms will be used interchangeably.
For the sake of notational simplicity, we assume that the arms are sorted by mean,
in that $\mu_1 \ge \hdots \ge \mu_n$, and that $\mu_k > \mu_{k+1}$, i.e. the top-$k$ are well defined.
We begin by observing that we do not exceed our allotted budget.

\begin{lem}
Algorithm \ref{alg:bandit_topk} does not exceed 
the budget $T$.
\end{lem}
\begin{proof}
At the $r$-th stage we have $|\CI_r|$ questions and $t_r$ workers.
Hence,
\aln{
\sum_{r=0}^{r_{\rm max}} |\CI_r| t_r 
& \leq 
\sum_{r=0}^{r_{\rm max}} \frac{T}{2 \left\lceil\log_2 \frac{n}{\sqrt T} \right\rceil}
= \frac{T (r_{\rm max} + 1)}{2 \left\lceil\log_2 \frac{n}{\sqrt T}\right\rceil}
= \frac{T \left\lceil\log_2 \frac{n}{\sqrt T}\right\rceil}{2 \left\lceil\log_2 \frac{n}{\sqrt T}\right\rceil}
= \frac{T}{2}
}
Since the clean up
stage uses at most $\frac{T}{2}$ pulls, the algorithm does not exceed its budget of $T$ pulls.
\end{proof}

We now examine one round of our adaptive spectral algorithm and bound the probability that the algorithm eliminates one of the top-$k$ arms in round $r$, recalling that

$$ t_r = \left\lfloor \frac{T}{2|\CI_r| \lceil 
\log_2 \frac{n}{\sqrt{T}} \rceil} \right\rfloor.$$

In standard bandit analyses, we obtain concentration of our estimated arm means via Hoeffding's inequality,
which we are unable to utilize here.
Theorem \ref{thm:final_confidence} states that 
\aln{
\P(|\hat{u}_i - u_i| \ge \epsilon)\le 3n\exp\left( - C_1 \epsilon^2 m\right),
}
providing a Hoeffding-like bound that allows us to eliminate suboptimal arms with good probability. 

\begin{lem}
The probability that one of the top $k$ arms is eliminated in round $r$ is at most 
$$18kn\exp \left(- C_5 \frac{\Delta_{i_r}^2}{i_r}  \frac{T  }{\log \frac{n}{\sqrt{T}}}\right)$$
for $i_r = |\CI_r|/4 = \frac{n}{2^{r+2}}$, and $C_5=\frac{C_1}{64}$.
\end{lem}
\begin{proof}
The proof follows similarly to that of \cite{Karnin2013}.
To begin, define $\CI_r'$ as the set of coordinates in $\CI_r$ excluding the $i_r = \frac{1}{4} |\CI_r|$ coordinates $i$ with largest $u_i$. 
Let $\hat{u}_i^{(r)}$ be the estimator of $u_i$ in round $r$.
We define the random variable $N_r$ as the number of arms in $\CI_r'$ whose $\hat{u}_i^{(r)}$ in round $r$ is larger than that of any of the top-$k$ $u_i$. 
We begin by showing that $\E [N_r]$ is small.
We bound $\E [N_r]$ as
\begin{align*}
    \E [N_r] &= \sum_{i \in \CI_r'} \P \left(\bigcup_{\ell \in [k]} \left\{\hat{u}_i^{(r)} \ge \hat{u}_\ell^{(r)} \right\}\right) 
    \le k\sum_{i \in \CI_r'} \P \left(\hat{u}_i^{(r)} \ge \hat{u}_k^{(r)}\right)\\
    &\le k\sum_{i \in \CI_r'} \P \left(\hat{u}_i^{(r)} \ge u_i+\Delta_i/2 \right) + \P\left(\hat{u}_k^{(r)}< u_k - \Delta_i/2\right)\\
    &\le k|\CI_r'| \left( \P \left(\hat{u}_{i_r}^{(r)} \ge u_{i_r}+\Delta_{i_r}/2 \right) + \P\left(\hat{u}_k^{(r)}< u_k - \Delta_{i_r}/2\right) \right)\\
    &\le 6k|\CI_r'| |\CI_r| \exp \left( - \frac{C_1}{4} \Delta_{i_r}^2 t_r\right)
    \\
    &\le 6k|\CI_r'| n\exp \left(- \frac{C_1}{64} \frac{\Delta_{i_r}^2}{i_r}  \frac{T  }{\log \frac{n}{\sqrt{T}}}\right),
\end{align*}

We note that since the $i_r$ largest entries of $\CI$ are not present in $\CI_r'$, we have that $\max_{i \in \CI_r'} u_i \le u_{i_r}$.
We now see that in order for one of the top $k$ arms to be eliminated in round $r$, at least $|\CI_r|/2$ arms must have had higher empirical scores in round $r$ than it. This means that at least $|\CI_r|/4$ arms from $\CI_r'$ must outperform the top $k$ arms, i.e., $N_r \ge |\CI_r|/4 = |\CI_r'|/3$.
Note that this analysis only holds when $|\CI_r| \ge 4k$.
We can then bound this probability with Markov's inequality as
\begin{align*}
    \P \left(\begin{array}{c|c}\text{At least one of top-$k$ arms}& \text{None of the top $k$ arms}\\
    \text{ eliminated in round $r$}& \text{eliminated till round $r$}
    \end{array}\right)
    & \le \P \left(N_r \ge \frac{1}{3}|\CI_r'|\right) \le 3 \E[N_r]/|\CI_r'| 
    \\
    &
    \le 18k n\exp \left(- \frac{C_1}{64} \frac{\Delta_{i_r}^2}{i_r}  \frac{T  }{\log \frac{n}{\sqrt{T}}}\right),
\end{align*}
concluding the proof of the lemma.
\end{proof}

\begin{lem}
The total probability of failure during the elimination 
stages, $P_{e}$, is bounded as
\begin{align*}
    P_{e} &\le 18kn \log\frac{n}{\sqrt{T}} \exp \left(- C_5 \frac{\Delta_{i_r}^2}{i_r}  \frac{T  }{\log \frac{n}{\sqrt{T}}}\right)
\end{align*}
\end{lem}

\begin{proof}
We see by a union bound over the stages that the probability that the algorithm fails (it eliminates one of the top $k$ arms) in any of the $\log \frac{n}{\sqrt{T}}$ halving stages is at most
\begin{align*}
    P_{e1} &= 
    \sum_{r=0}^{\log\frac{n}{\sqrt{T}}-1}\P\left(\begin{array}{c|c}\text{At least one of top-$k$ arms}& \text{None of the top $k$ arms}\\
    \text{ eliminated in round $r$}& \text{eliminated till round $r$}
    \end{array}\right) \P \left(\begin{array}{c} \text{None of the top $k$ arms}\\ \text{eliminated till round $r$}
    \end{array}\right) \\
    &\le \sum_{r=0}^{\log\frac{n}{\sqrt{T}}-1}\P\left(\begin{array}{c|c}\text{At least one of top-$k$ arms}& \text{None of the top $k$ arms}\\
    \text{ eliminated in round $r$}& \text{eliminated till round $r$}
    \end{array}\right)\\
     &\le\sum_{r=0}^{\log\frac{n}{\sqrt{T}}-1}  18kn\exp \left(- C_5 \frac{\Delta_{i_r}^2}{i_r}  \frac{T  }{\log \frac{n}{\sqrt{T}}}\right)\\
     &\le \sum_{r=0}^{\log\frac{n}{\sqrt{T}}-1}  18kn\exp \left(- C_5   \frac{T  }{ \log \frac{n}{\sqrt{T}} \cdot \max_s \frac{i_s}{\Delta_{i_s}^2} }\right) \\
     &\le  18kn \log\frac{n}{\sqrt{T}} \exp \left(- C_5   \frac{T  }{ \log \frac{n}{\sqrt{T}} \cdot \max_{i \ge \sqrt{T}} \frac{i}{\Delta_{i}^2} }\right)\\
     &\le  18kn \log\frac{n}{\sqrt{T}} \exp \left(- C_5 \frac{T  }{H_2\log \frac{n}{\sqrt{T}}}\right),
\end{align*}
as $H_2 \triangleq \max_i \frac{i}{\Delta_{i}^2}$ .
This concludes the proof of the lemma.
\end{proof}

\begin{lem}
The total probability of failure during the clean up  
stage, $P_f$, is upper bounded as
\begin{align*}
    P_f &\le 12T \exp \left( -\frac{C_1}{16}\Delta_{k+1}^2 \sqrt{T} \right)
\end{align*}
\end{lem}
\begin{proof}

Through our halving stages, we are left with at most $2\sqrt{T}$ active coordinates. We now use a budget of $T/2$, i.e. $m\ge \sqrt{T}/4$ columns to estimate their means. Then, the probability that the top $k$ are not the true top $k$ entries (given that none of the top $k$ were eliminated previously) is:
\begin{align*}
    P_f
    &\le \sum_{i \in \CI_{r_{max}+1}} \P \left( |\hat{u}_i - u_i)| > \Delta_{k+1}/2\right)\\
    &\le 2\sqrt{T} \cdot \P \left( |\hat{u}_1 - u_1)| > \Delta_{k+1}/2\right)\\
    &\le 12T \exp \left( -\frac{C_1}{16}\Delta_{k+1}^2 \sqrt{T}\right)
\end{align*}
\end{proof}
Thus, our overall error probability is at most
\begin{align*}
    \P(\text{failure}) &\le P_e + P_f\\
    &\le 18kn \log\frac{n}{\sqrt{T}} \exp \left(- C_5 \frac{T  }{H_2\log \frac{n}{\sqrt{T}}}\right) + 12T \exp \left( -\frac{C_1}{16}\Delta_{k+1}^2 \sqrt{T}\right) \\
    &\le 18kn \log n \exp \left(- C_5 \frac{T  }{H_2\log n}\right) + 12n^2 \exp \left( -\frac{C_1}{16}\Delta_{k+1}^2 \sqrt{T}\right)
\end{align*}
with budget no more than T.

Inverting this, we have that for a probability of error $\delta$,
one needs $T=O\left( H_2\log n \log \left( \frac{kn \log n }{\delta} \right)  + 
\Delta_{k+1}^{-4} \log^2\left(\frac{n^2}{\delta} \right)
\right)$, giving us the result claimed.

\section{Proof of Theorem \ref{thm:thresh_bdt_main}}\label{app:threhsold}

In this appendix we provide the proof of Theorem \ref{thm:thresh_bdt_main} regarding the performance of our thresholding bandit algorithm.
\begin{proof}
We note that while in the elimination stage, we always 
have $t_r \le |S_r|$ by construction, as $t_0 \le |S_0|$
and $t_r$ is an increasing sequence while $|S_r|$
is a decreasing one, and we break when 
$|S_r| < \frac{12}{(\beta-\alpha)^2 C_{4}} \log n$, 
while the maximum $t_r$ achieved in the \textbf{for} 
loop of line $r$ is $\frac{12}{(\beta-\alpha)^2 C_{4}} \log n$
because of the limits of the for loop.
Further $t_r$ is 
picked so that in round $r$, with probability $1-\frac{1}{n^2}$,
$|\hat{u}^{(r)}_i- u_i| \le 2^{-r} $ 
for all $i \in \{1,..., n\}$. 
Notice that there are at most
$\lceil \log_2 \frac{1}{\beta - \alpha} \rceil$ iterations and, since
$\beta - \alpha$ is assumed to be a constant, 
$\lceil \log_2 \frac{1}{\beta - \alpha} \rceil \leq n$ for large enough $n$.
Hence, with probability at least $1-\frac{1}{n}$,
we have that
over all rounds, $\hat{u}^{(r)}_i$ are 
within their confidence intervals for all coordinates $i$.
Similarly the clean up stage is constructed so that
$|\hat{u}^{(r)}_i- u_i| \le \beta -\alpha$ with 
probability at least $1 - \frac{1}{n^2}$. 
Thus with probability at least $1-\frac{2}{n}$ all our
estimates are within the confidence intervals constructed
throughout the algorithm.

Notice that if the confidence interval of question $i$ (with parameter $u_i$) is reduced to less than $\Gamma_i/2$ in the $r$-th iteration, at that point (or previously) the algorithm must either accept or reject $u_i$.
This is because any $u_i > \beta$ will have $\hat u_i^{(r)} - \Gamma_i/2 > \alpha$, any $u_i < \alpha$ will have $\hat u_i^{(r)} + \Gamma_i/2 < \beta$, and any $u_i \in [\alpha,\beta]$ will have a confidence interval of total length at most $\Gamma_i = \beta - \alpha$, which cannot include both $\alpha$ and $\beta$.
Hence, for each question $i$ of the $n- \kappa$ questions eliminated before the clean up stage, the total number of workers used at its last iteration before elimination is, by Lemma~\ref{lem:CI}, at most 
\aln{
\frac{3 \log n + 2 \log (2n)}{C_4 (\Gamma_i/2)^2} 
\leq \frac{24 \log n}{C_4 \Gamma_i^2},
}
where we upper bound $2\log(2n)$ by $3\log(n)$ for $n>4$.
Since the number of workers used at each iteration grows as a geometric progression, the total number of questions answered for all the $n - \kappa$ questions that are eliminated before the clean up stage is at most
\begin{align*}
    \sum_{\ell=\kappa +1}^n \frac{32}{C_4}\frac{\log n}{(\Gamma^{(\ell)})^2}.
\end{align*}
Similarly for each question resolved in the clean up stage of $\kappa$ workers, we 
need at most $\frac{24}{C_4}\frac{\log n}{(\beta-\alpha)^2}$ total
worker responses. 
Since
\begin{align*}
 \kappa \frac{24}{C_4}\frac{\log n}{(\beta-\alpha)^2} 
 \leq 2\left( \frac{12}{C_4}\frac{\log n}{(\beta-\alpha)^2}\right)^2,
\end{align*}
the result follows.
\end{proof}

\section{Constants}\label{app:constants}

The results in Section~\ref{sec:analysis} are stated in terms of several constants. We assume that there is some $c> 0$ such that
$u_i> c \ \ , q_j > c$ for all $i, j$.
Following the derivations in their respective proofs, these constants are given by
\aln{
    C_1 &= \min(C_4,(6C_3/c)^{-2})\\
    C_2 &= c^4/48\\
    C_3 &= 4/c^4 + 30\sqrt{2}\\
    C_4&= c^2\min(1/18, C_2/9)\\
    C_5&=\frac{C_1}{64}.
}
We present these constants here for the sake of completeness, noting that several of the bounding steps in the derivations could be loose, and these constants are not expected to be tight.
One way to improve algorithm performance in practice is to first run the algorithm in Section~\ref{sec:analysis} on a dataset with known ground truth, 
and empirically estimate the true constants.

\section{Comparison with other methods of constructing confidence intervals}

In this section, we discuss two alternative ways to construct 
estimators to be used with the bandit algorithms. 
We first consider row averages as an estimator for the bandit 
algorithms, and then discuss the connection of bounds we 
use on our spectral estimators with those of \citet{abbe2017entrywise}.

\subsection{Row averages}
In this section, we discuss an alternative way to construct 
estimators to be used with the bandit algorithms. 
We consider row averages as an estimator for the bandit 
algorithms and show that we could run a top-$k$
algorithm with row averages as the estimators, but would
not be able to run the thresholding bandits algorithm as we do not
obtain unbiased estimates.

We can construct estimators based off of row sums.
For $X$ with $\E X = \u\v^\top$, we estimate $u_i$ with
\begin{equation}
    \hat{u}_i^{(m)} = \frac{1}{m}\sum_{j=1}^m X_{i,j}.
\end{equation}
Notice that this is similar in spirit to the Jaccard similarity estimator described in (\ref{eq:mhap}) and, in practice, provides a worse estimator to the overlap sizes than the estimators based on a rank-one model \cite{baharav2019spectral}.
However, these estimators can theoretically still be used to find the reads with the largest overlaps, as we describe next.
Considering that $X_{i,j}$ has expectation $u_iv_j$, we note that 
\begin{align*}
    \hat{u}_i^{(m)} - \hat{u}_k^{(m)}
    &= \frac{1}{m}\sum_{j=1}^m \left( X_{i,j} - X_{k,j}\right)\\
    &= \frac{1}{m}\sum_{j=1}^m \left( (X_{i,j}-u_iv_j)+u_iv_j - (X_{k,j}-u_kv_j) - u_kv_j\right)\\
    &= \frac{1}{m}\sum_{j=1}^m \left( X_{i,j}-u_iv_j-(X_{k,j}-u_kv_j)\right)+\frac{1}{m}\sum_{j=1}^m v_j(u_i -u_k)
\end{align*}
Hence, for $u_i > u_k$,
\begin{align*}
    \P\left( \hat{u}_i^{(m)} < \hat{u}_k^{(m)} \right)
    &\le \P \left( \frac{1}{m}\sum_{j=1}^m \left( X_{i,j}-u_iv_j-(X_{k,j}-u_kv_j)\right) < \frac{1}{m}\sum_{j=1}^m v_j(u_k -u_i)\right)\\
    &\le \P \left( \frac{1}{2m}\sum_{j=1}^m \left( X_{i,j}-u_iv_j-(X_{k,j}-u_kv_j)\right) <  \left(\frac{1}{m}\sum_{j=1}^m v_j \right)\frac{u_k -u_i}{2}\right)\\
    &\le 2\exp \left(-  \frac{m}{4}\left(\frac{1}{m}\sum_{j=1}^m v_j \right)^2 (u_k - u_i)^2 \right),\\
    &= 2\exp \left(-  \frac{m\bar{v}^2 (u_k - u_i)^2}{4} \right),
\end{align*}
where $\bar{v} = \frac{1}{m}\sum_{j=1}^m v_j$.

This follows since $X_{i,j}-u_iv_j$ is a zero-mean bounded random variable.
This bound implies that 
$m = \frac{4\log \left(\frac{2n}{\delta} \right)}{ \bar{v}^2\epsilon^2}$ yields the desired result, that $\P\left( \argmax_{i \in [n]} \hat{u}_i^{(m)} =  \argmax_{i \in [n]} u_i\right) \ge 1-\delta$.
and so for the top-$k$ scenario we have that a budget of 
\begin{equation}
    T=n\frac{\log \left(\frac{2n}{\delta} \right)}{ \bar{v}^2\Delta_{k+1}^2}
\end{equation}
is required by the uniform sampling row sum algorithm.
We note that this concentration analysis is for uniform sampling, but it shows us that  we could do sequential halving on the row sums to adaptively find the maximal $u_i$, with a similar analysis to \cite{baharav2019ultra}.
Note that this is critically using the fact that the estimators for $u_i$ and $u_k$ are taken across the same $v_j$, and that we are not able to generate unbiased estimates of the $u_i$ with this method, only to preserve ordering. Hence while such an estimator can 
be used with a top-$k$ bandits algorithm, we are unable 
to use it with a thresholding bandit algorithm.

\begin{note}
To provide some intuition, we remark that the row averages estimator has an
advantage, in that for an $n \times m$ matrix $X$ the 
error of  $\hat{u}_i^{(m)} - \hat{u}_k^{(m)}$
decays roughly as $\frac{1}{\sqrt{m}}$ (when $\bar{v}$ is $O(1)$). On the other hand, with the 
spectral estimator we use, the error 
of the estimator of ${u}_i$ only decays as
roughly $\frac{1}{\sqrt{\min(n,m)}}$.
\end{note}

\subsection{Other spectral estimators}\label{app:abbe}
In \citet[Eq (2.7)]{abbe2017entrywise}, the authors consider the 
case of symmetric matrices $X$. For the case when $X$ is rank $1$
their result can be interpreted as
\begin{align}\label{eq:abbe}
    \left\|\hat{\mathbf{u}} - \frac{X \u}{\|\u\|^2}\right\|_{\infty} = o_{\mathbb{P}}(1) \|\u\|_{\infty},
\end{align}
where $o_{\mathbb{P}}$ is the small oh notation for convergence in probability. 
In particular, a sequence of random variables $\{Z_n\}_{n\ge 1}$ is $ o_{\mathbb{P}}(1)$
if $\lim_{n \rightarrow\infty }\mathbb{P}(|Z_n| > \epsilon) = 0$ for all $\epsilon > 0$.

They note that one can control this quantity even though there
are simple examples where $\|\hat{\u} - \u\|_{\infty}$ 
is not controlled. In this manuscript we deal with the rectangular
matrix and derive control over the slightly different quantity,
\begin{align}
    \left\|\u - \frac{X \hat{\v}}{\|\v\|\|\hat{\v}\|}\right\|_{\infty}
\end{align}
for the special case when $\E X$ is rank $1$.
This is similar in spirit to their work as we do not directly 
control the $\ell_{\infty}$ norm of the differences of 
the observed and expected left singular vector.
However, we note that controlling the LHS of \eqref{eq:abbe}
does not trivially give us the confidence intervals we need.

\section{Implementation Details for Algorithm~\ref{alg:bandit_topk}}\label{app:implementation}
While in theory we stop at $\sqrt{T}$ arms remaining, in practice we continue halving until there are fewer than $2k$ remaining arms, at which point we output the $k$ arms with largest $\hat{u}_i$.
While theoretically we are unable to take advantage of the scenario $m>n$ (due to the constraint in Theorem~\ref{thm:final_confidence}), in practice, increasing $m$ beyond $n$ still improves the estimates $\hat u_i$, and we do not need to perform the clean up stage with $\sqrt{T}$ arms remaining.

In the practical implementation of Algorithm~\ref{alg:bandit_topk}, we also impose a maximum number of measurements per item (finite number of workers), and so terminate our algorithm and return the top $k$ if $t_r=m_{\rm max}$ for some a priori fixed quantity $m_{\rm max}$.

While Algorithm~\ref{alg:bandit_topk} requires 
``oracle knowledge'' of $\|\v^{(r)}\|$, in practice that cannot be obtained 
and we use $\|\hat{\v}^{(r)}\|$ instead.
Notice that knowledge of the exact value of $\|\v\|$ would only provide a rescaling of our estimates, and so relative ordering is preserved in our $\hat{\u}^{(r)}$ if we use $\|\hat{\v}^{(r)}\|$, which is sufficient for top-$k$ identification. 
Notice that this is not the case for the thresholding bandits considered in Appendix~\ref{app:threhsold}.
For the \textit{E.~coli} dataset, we obtain $\hat{\v}^{(r)}$ in Algorithm \ref{alg:spectral} using the scheme proposed in \citet{baharav2019spectral}, that is by taking column sums of $X^{(r)}$. 
We run our simulations on reference read 1 of their dataset, as reference read 0 has 10 non-trivial alignments, whereas reference read 1 has 5 as desired.

While in theory we do not reuse old samples to maintain independence, in our algorithm we do.
This is done naturally by running Algorithm 1 on the $\{0,1\}^{|\CI_r|\times \sum_{i=0}^r t_r}$ matrix of responses on all the previously asked questions (not just those in the current round).
Similarly, we do not split our matrix in 2 for Algorithm \ref{alg:spectral}; we estimate $\hat{\v}^{(r)}$ from the entirety of $X^{(r)}$, and compute $\hat{\u}^{(r)}=X^{(r)} \hat{\v}^{(r)}$.

For top 2k, we ran both algorithms to return their estimated top 2k, which we denote as the set \texttt{Top-2k}, then evaluated the performance as $\frac{1}{k} \sum_{i=1}^k \one\{(i) \in \texttt{ Top-2k} \}$.
For the error probability plot, we evaluated the performance by running the algorithms to return their estimated top k, \texttt{Top-k}, and computing $\one \left\{\texttt{Top-k} \equiv \{(i) : i \in [k]\} \right\}$.
For each simulation, we run 100 trials and report the mean of our performance metric as well as its standard deviation (shaded in).

For the controlled experiments, $u_i$'s are generated according to a Beta$(1,5)$
distribution and $q_i$ are generated according to a uniform$(0,1)$ distribution.

\subsection{Computing architecture and runtime}
Min-hashes took 18 hours to generate on 50 cores of an AMD Opteron Processor 6378 with 500GB memory. 
Generating the empirical results for uniform and adaptive on the \textit{E.~coli} dataset took took 36 minutes on one core.
Generating the empirical results for the synthetic crowdsourcing experiments took 3 hours on one core, due to the fact that there is no efficient approximation for the right singular vector, and so one needs to compute the actual SVD of $X^{(r)}$ in every iteration.

\end{document}